\documentclass{article}
\usepackage{amssymb}
\usepackage{amsthm}
    \PassOptionsToPackage{compress, numbers}{natbib}
    \PassOptionsToPackage{colorlinks=true, citecolor=blue}{hyperref}

     \usepackage[preprint]{neurips_2020}

\usepackage[utf8]{inputenc} 
\usepackage[T1]{fontenc}    
\usepackage{hyperref}       
\usepackage{url}            
\usepackage{booktabs}       
\usepackage{amsfonts}       
\usepackage{nicefrac}       
\usepackage{microtype} 
\title{Contextualizing Enhances \\ Gradient Based Meta Learning}

\usepackage{todonotes}
\usepackage{textcomp}
\usepackage{amsmath}
\usepackage{colonequals}

\newtheorem{prop}{Proposition}

\newcommand{\bd}[1]{\boldsymbol #1}

\author{%
  Evan~Vogelbaum\thanks{Equal contribution.} \\
  MIT \\
  \texttt{evanv@mit.edu} \\
   \And
    Rumen Dangovski$^*$ \\
   MIT \\
   \texttt{rumenrd@mit.edu} \\
   \AND
   Li Jing \\
   MIT \\
   \texttt{ljing@mit.edu} \\
   \And
   Marin Solja\v{c}i\'c \\
   MIT \\
   \texttt{soljacic@mit.edu} \\
}

\begin{document}

\maketitle

\begin{abstract}
Meta learning methods have found success when applied to few shot classification problems, in which they quickly adapt to a small number of labeled examples. Prototypical representations, each representing a particular class, have been of particular importance in this setting, as they provide a compact form to convey information learned from the labeled examples. However, these prototypes are just one method of representing this information, and they are narrow in their scope and ability to classify unseen examples. We propose the implementation of \emph{contextualizers}, which are generalizable prototypes that adapt to given examples and play a larger role in classification for gradient-based models. We demonstrate how to equip meta learning methods with contextualizers and show that their use can significantly boost performance on a range of few shot learning datasets. We also present figures of merit demonstrating the potential benefits of contextualizers, along with analysis of how models make use of them. Our approach is particularly apt for low-data environments where it is difficult to update parameters without overfitting. Our implementation and instructions to reproduce the experiments are available at \url{https://github.com/naveace/proto-context/}.

\end{abstract}

\section{Introduction}


With the rise of deep learning, models have become remarkably successful at mastering challenging, large scale image classification tasks by training on enormous amounts of data~\citep{NIPS2012_4824, guo2016deep,  litjens2017survey}. However, this reliance on ``big data'' remains a key flaw in most deep learning models, limiting their use cases and making them expensive to train \citep{jin2016scale, peng2018macro, kolavr2016deep}. To address this problem, many have turned to few shot learning methods: algorithms and architectures designed to perform tasks such as image classification with very small amounts of task-specific data. Much progress has been made in this area \citep{vinyals2016,triantafillou2017fewshot,snell-etal-2017-prototypical,  Sung_2018},
however many of these models nonetheless remain specific to the tasks for which they are designed to solve \citep{ren2018meta, li2017metasgd, kang2018transferable, UnifiedApproachRahman}.

Another method of addressing the ``big data'' problem with more generalizability has been found in \emph{meta learning}, a field very closely related to few shot learning, in which models and algorithms are designed to quickly fine-tune on new tasks either through transferring knowledge from previously seen data samples or optimizing themselves for fast adaptation in just a few gradient updates \citep{vanschoren2018metalearning}. An example of the latter is \emph{model agnostic meta learning} (MAML), an algorithm designed by \citet{finn2017}  that excels at training neural networks for meta learning problems, and which was improved upon by \citet{nichol2018firstorder, antoniou2019, behl2019alpha, song2019esmaml} and many more.

\emph{Metric learning}, such as that proposed by \citet{snell-etal-2017-prototypical}, has also been shown to be successful in the area of few shot learning. \emph{Prototypes}---averages of examples that share the same class---were used by~\citet{snell-etal-2017-prototypical} to make classification decisions by measuring how close unclassified samples are to each prototype. \citet{triantafillou2019} helped marry prototypes with MAML through \emph{ProtoMAML}, a version of MAML that initializes the \emph{head} of the network (the final layer that maps features to output classes) with scaled versions of the prototypes, and as a result it is end-to-end trainable with MAML. 

These methods are promising; however, current models that combine few shot learning with meta learning fail to produce \emph{task specific} initializations for heads. In addition, although the ideas of task-specific feature spaces have been explored \citep{oreshkin2018tadam, perez2017film}, these have yet to be combined with the advances in head initialization proposed in \citep{triantafillou2019}. Finally, most meta-learning models rely on several gradient steps to adapt the model to a particular task and can be prone to overfitting, a problem that has gained attention before \citep{antoniou2019, zintgraf2018fast, li2017metasgd, raghu2019rapid}. We address these problems through the use of ``contextualizers:'' generalized prototypes that produce (\emph{i}) a task-specific initialization (\emph{ii}) a task-specific feature space and (\emph{iii}) achieve near peak accuracy in as little as one inner loop gradient step. 


In this paper, we focus on applying contextualization to the problem of few shot classification. This setup consists of training on separate \emph{tasks}, where each task consists of a \emph{support} set $\mathbf{S}=\left(\mathbf{X},\mathbf{y}\right)\equiv \left\{ (\mathbf{x}_j, y_j)\right\}_{j=1}^{|\mathbf{S}|}$ and a \emph{query} set $\mathbf{Q}=\left(\mathbf{X}',\mathbf{y}'\right)\equiv \left\{ (\mathbf{x}_j', y_j')\right\}_{j=1}^{|\mathbf{Q}|}$ of images and labels~\citep{snell-etal-2017-prototypical}. 
Each task consists of \emph{n} different classes with \emph{k} examples per class in the support set. This type of learning is referred to as $n$-\emph{way} $k$-\emph{shot}, and we will use this terminology 
throughout~\citep{vinyals2016}.
We fine-tune the model on the samples in $\mathbf{S}$ and then test the fine-tuned model on the samples in $\mathbf{Q}$. The performance on $\mathbf{Q}$ is used to \emph{meta train}
the model, which in the \emph{second order} version of MAML consists of propagating the loss through the update made from the gradient on $\mathbf{S}$ to the base parameters, and thus requires the calculation of Hessians (second order gradients). This method of updating can be impractical due to size and time constraints; hence, we focus on the \emph{first order} setting of MAML~\citep{finn2017}. In this setting, we begin with a model at a set of \emph{base parameters}. Next, we fine-tune the model through gradient descent on $\mathbf{S}$ and then evaluate it on $\mathbf{Q}$.
Finally, we apply the gradient of the loss on $\mathbf{Q}$ with respect to the tuned parameters directly to the base parameters:
\begin{align*}
    L_\mathbf{S}\left(\bd{\theta}^{(0)}\right) & = L\left(\bd{\theta}^{(0)}, \mathbf{S}\right)  \tag*{loss on the support $\mathbf{S}$}\\
    \bd{\theta}^{(i+1)} & = \bd{\theta}^{(i)} - \alpha \cdot \nabla_{\bd{\theta}^{(i)}} L_\mathbf{S}\left(\bd{\theta}^{(i)}\right) \tag*{inner loop steps $i=1,\dots,M$} \\
    L_\mathbf{Q}\left( \bd{\theta}^{(M)}\right) & = L\left(\bd{\theta}^{(M)}, \mathbf{Q}\right) \tag*{loss on the query $\mathbf{Q}$} \\
    \bd{\theta}^{(0)} & = \bd{\theta}^{(0)} - \beta \cdot \nabla_{\bd{\theta}^{(M)}} L_\mathbf{Q}\left(\bd{\theta}^{(M)}\right), \tag*{outer loop update of initial weights $\bd{\theta}^{(0)}$}
\end{align*}
where $L$ is our loss function, $\bd{\theta}^{(0)}$ are the base parameters,
 $M$ is the number of times we update our parameters on the support set (this is called \emph{inner loop} adaptation), and $\alpha$ and $\beta$ are our inner loop and base update (\emph{outer loop}) learning rates.
 \citet{finn2017} showed that first order MAML often performs comparably to the second order version.
 We introduce an augmentation to the forward pass of a network trained using MAML that allows for the creation of task-specific intializations for the \emph{head} of the model as well as the creation of task-specific feature spaces. Our main contributions are: 
\begin{enumerate}
    \item We propose a new method of meta learning that combines task-specific feature spaces with task-specific head initializations and demonstrate its ability to outperform conventional gradient methods on a range of few shot learning datasets.
    \item We present empirical analysis of our contextualization mechanism demonstrating that our contextualizers play a major role in the model.
    \item We introduce a new figure of merit, \emph{intra-class similarity}, to measure upstream benefits our contextualizers may have in the training process.
\end{enumerate}


\section{The Concept of Contextualization}
Contextualization can be thought of as adapting a model's features and initialization of parameters to a task without a gradient update. We present the general mechanism of contextualization for a model's features below and in Figure~\ref{imageprotocontext}, and then detail how it can be used to construct task-specific initializations.
\begin{figure}[t]
\centering

 \includegraphics[width=\textwidth]{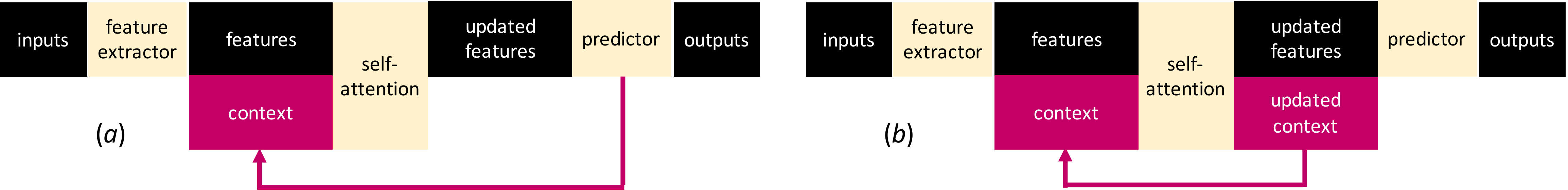}
\caption{We propose two forms of contextualization for the features. One uses the weights of the head of the model to contextualize each forward pass (\emph{a}), the other uses a set of context vectors continuously updated through self-attention to attend to the features (\emph{b}).}
 \label{imageprotocontext}
\end{figure}

\paragraph{Model} We denote our feature extractor (i.e. convolutional neural networks + flattening) by $f_{\bd{\theta}}$, parameterized by the weights $\bd{\theta},$ the contextualization mechanism by $s_{\bd{\phi}}$ (we use a self-attention mechanism \citep{vaswani2017attention} in our model), parametrized by the weights $\bd{\phi}$, the predictor function (head) by $p_{\bd{\psi}}$, parametrized by the weights $\bd{\psi},$ and the loss as $L$.


\paragraph{Contextualization for features}
Previous work has theorized that the $f_{\bd{\theta}}$ trained by MAML produces highly generalizable features~\citep{raghu2019rapid}. While this is an appealing property, we believe that performance can be boosted by adapting these generalized features to be specific to the task at hand. If this is done in conjunction with adaptation of the head, one can produce a highly task-specific model even without gradient updates.

We evaluate two forms of feature adaptation (which we call contextualization) for few shot classification. The first (Figure~\ref{imageprotocontext}\emph{a}) uses the weights of the head of the network as contextualization, with the understanding that samples which activate highly with a particular layer of the head are likely to belong to the class that layer of the head connects to. In this case, the contextualization is updated through the gradient descent of the head. We also explore a method of contextualization that allows for the context to be updated in a non-gradient manner completely independent of the head. This contextualization (Figure~\ref{imageprotocontext}\emph{b}) begins with a set of contextualized prototypes (described in more detail below) and continuously updates them through self-attention during each forward pass. Our hypothesis is that with this contextualization, the context will become a non-gradient method of transferring information from the support set to the query set.

Regardless of the form of contextualization we choose, both follow the same forward pass. We describe the algorithm at a high level here, and point the reader to our Supplementary Materials (SM) for the specific equations detailing contextualization. Images go through $f_{\bd{\theta}}$. We concatenate the extracted features with a set of ``contextualizers'' in the same feature space and feed them through $s_{\bd{\phi}}$. This mechanism updates the features, creating a task-dependent feature space. The updated features are then fed through $p_{\bd{\psi}}$ to produce outputs.

\section{Head Initializations}
In this section, we (\emph{i}) conduct a simple study to demonstrate the importance of intitializing the head properly for each task, (\emph{ii}) introduce the use of prototypical representations for head initialization, and (\emph{iii}) in light of (\emph{i, ii}) we marry our concept of contextualization with that of prototypes.

\paragraph{Head initialization is important} In Table~\ref{tab:simple}  we present the \emph{meta test} results for a simple study that aims to understand how important the initialization for the head is. The Regular experiment is MAML. In the second experiment, we train with MAML and randomize the head before the inner loop during meta testing. In the last experiment, during both meta training and testing we start the inner loop with a random initialization for the head to see if the network can learn to overcome the random initialization. We observe significant degradation in accuracy as the meta initialization is lost via randomization. These results suggest that fine-tuning through the inner loop is not enough to produce a strong task-specific head, and thus that initialization matters. We use these insights to motivate our exploration of new ways to initialize the head for each task.

\begin{table}[th]
  \small
  \centering
  \caption{MAML with different initialization schemes for the head of the network. Accuracy in \%. Experiments performed with five inner loop steps. Mean and standard deviation from three trials.}
  \begin{tabular}{lccc}
    \toprule
    \multicolumn{1}{c}{Experiment} &
    \multicolumn{1}{c}{Regular} &
    \multicolumn{1}{c}{Random Head (test only)}
    & \multicolumn{1}{c}{Random Head (train \& test)} \\
    \cmidrule(r){1-1}
    \cmidrule(r){2-2}
    \cmidrule(r){3-3}
    \cmidrule(r){4-4}
    Mini-ImageNet (5-way 1-shot) & $49.16 \pm 0.27$ & $34.89 \pm 1.00$ & $34.96 \pm 0.95$ \\
    Mini-ImageNet (5-way 5-shot) & $66.61 \pm 0.55$ & $45.13 \pm 0.65$ & $59.60 \pm 0.34$ \\
    \bottomrule
  \end{tabular}
  \label{tab:simple}
\end{table}

\paragraph{Prototypes as initializations}
ProtoMAML builds on MAML by bringing the concept of prototypes 
into the MAML training process. Prototypes are vectors meant to be representative of a particular class, and are computed by averaging together the features of support samples for a given class $a$
\begin{align}
    \label{eq:computing_prototypes}
    \mathbf{X}_{(a)} & = \left \{\mathbf{x}_{j}^{(a)}|1\leq j \leq k, y_j = a \right \} \subseteq \mathbf{S} \\
    \label{eq:prototypes}
    \mathbf{p}^{(a)} & = \frac{1}{k} \sum_{\mathbf{x} \in \mathbf{X}_{(a)}} f_{\bd{\theta}}(\mathbf{x}). 
\end{align}
ProtoMAML then initializes the head of the network with weights equal to $2 \mathbf{p}^{(a)}$ and biases equal to $-\left\|\mathbf{p}^{(a)}\right\|^2_2$ for $a=1,\dots n$~\citep{triantafillou2019, snell-etal-2017-prototypical} as follows:
\begin{align}
    \mathbf{H}_w = \left[2\mathbf{p}^{(1)}, \dots ,2\mathbf{p}^{(n)}\right]^\top & & 
    \mathbf{H}_b = \left[-\left\|\mathbf{p}^{(1)}\right\|^2_2, \dots , -\left\|\mathbf{p}^{(n)}\right\|^2_2\right]^\top.
    \label{eq:headinit}
\end{align}

\begin{figure}[t]
\centering

 \includegraphics[width=0.7\textwidth]{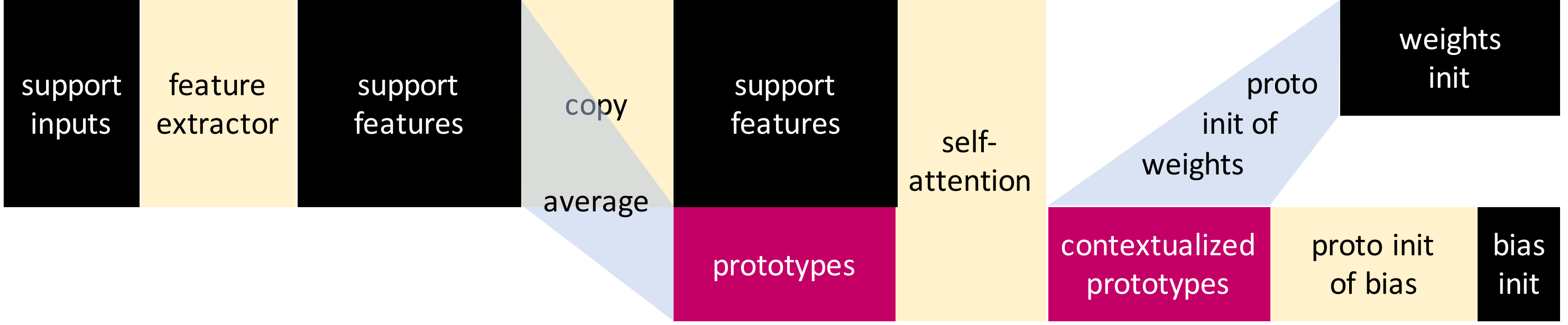}
 \caption{Contextualized prototypes form the initialization $\bd{\psi}$ of the predictor $p_{\bd{\psi}}$.}
 \label{contextinit}
\end{figure}

\paragraph{ProtoContext}

In Figure~\ref{contextinit} we introduce \textit{ProtoContext}:
our method of generalizing prototypes which can leverage ProtoMAML's initialization to produce even stronger, task-specific head initializations. 
 ProtoContext collects prototypes through Equation~\eqref{eq:computing_prototypes} and Equation~\eqref{eq:prototypes},
contextualizing them via self-attention \emph{using the support examples}  as follows:
\begin{align*}
    \mathbf{C} = s_{\bd{\phi}}\left(\mathbf{P}, f_{\bd{\theta}}\left(\mathbf{S}\right)\right),
\end{align*}
where $\mathbf{C}$ is a matrix with the contextualized prototypes for each class $\mathbf{c}^{(a)}$ and $\mathbf{P}$ is a matrix of the prototypes $\mathbf{p}^{(a)}$. Feeding the prototypes through the contextualization mechanism \emph{before} initialization ensures that the head is initialized with information about the task at large, and not just one particular class. We initialize $\mathbf{H}$ in the manner detailed in Equation~\eqref{eq:headinit}, using $\mathbf{c}^{(a)}$ instead of $\mathbf{p}^{(a)}$. 

With this head initialization we then intitalize our contextualization vectors using one of the two methods described above as shown in Figure~\ref{imageprotocontext}. Because the head is initialized using the contextualized prototypes, our initial context vectors are proportional to the contextualized prototypes regardless of the feature contextualization method we use:
\[
\mathbf{C}^{(0)}\propto \left[\mathbf{c}^{(1)},\dots,\mathbf{c}^{(n)}\right]^\top.
\]
We use this method of head initialization regardless of our contextualization for the features.


\paragraph{Why marry contextualization and the head?} 
This last fact, that our initial context vectors are proportional to the contextualized prototypes, is very useful as it provides a mechanism by which our model can leverage self-attention to boost prediction.  
In the beginning of the inner loop we obtain $\widetilde{\mathbf{X}}= s_{\bd{\phi}}(f_{\bd{\theta}}\left(\mathbf{X}),\mathbf{C}^{(0)}\right)$ where $\widetilde{\mathbf{X}} = \left\{ \widetilde{\mathbf{x}}_j\right\}$ is the collection of contextualized features. The contextualization by the self-attention constructs $\widetilde{\mathbf{x}}$ as a linear combination of features and contextualized prototypes: 
\[
\widetilde{\mathbf{x}} \propto \sum_{j}u_jf_{\bd{\theta}}(\mathbf{x}_j)+ \sum_{a}v_a \mathbf{c}^{(a)} +f_{\bd{\theta}}(\mathbf{x}) ,
\]
where $u_j$ and $v_a$ are the weights given by the self-attention. Using Equation~\eqref{eq:headinit} for our head initialization the predicted logit for class $a'$ is given by the $a'$-th component of the output logits as follows:
\begin{equation}
\label{eq:expansion}
\left(\mathbf{H}_w\widetilde{\mathbf{x}} + \mathbf{H}_b\right)_{a'} \propto f_{\bd{\theta}}(\mathbf{x}) \cdot 2\mathbf{c}^{\left(a'\right)} + \sum_{j}2u_j \mathbf{c}^{\left(a'\right)} \cdot f_{\bd{\theta}}(\mathbf{x}_j) + \boxed{\sum_a 2v_a \mathbf{c}^{\left(a'\right)} \cdot \mathbf{c}^{(a)}} -\left\|\mathbf{c}^{\left(a'\right)}\right\|^2_2.
\end{equation}
The boxed term demonstrates why contextualization through self attention can be effective when combined with Equation~\eqref{eq:headinit}. Each coefficient $v_a$ is determined by the self-attention key of $\mathbf{x}$ and the self-attention query of $\mathbf{c}^{(a')}$. Thus, the activation for class $a'$ can be strengthed by placing more weight on the contextualization vector $a'$ in our self-attention. In Section~\ref{sec:analysis} we demonstrate empirical validation of this technique in action by analysing the attention patterns produced by our experiments.

We also believe contextualization may have positive effects on the updates to the network during the inner loop. We present a step towards theoretical analysis of the network's updates with details of this theory in the SM.

\section{Experiments}


In our experiments we probe the extreme scenario of allowing only one gradient step to adapt to $\mathbf{S}$ before classifying $\mathbf{Q}$. 
We test on four few shot learning datasets and compare to the gradient methods of MAML and ProtoMAML as well as a non-gradient based method in Prototypical Networks~\citep{snell-etal-2017-prototypical}.
We test two forms of contextualization, one in which the contextualizer is constantly set to be the head of the model, and the other in which the contextualizer is initialized with the contextualizations of the prototypes and then updated via self-attention independent of the gradient. This method allows the contextualizer to serve as a special component of our model purely for attention, rather than as another trainable parameter like in \citep{zintgraf2018fast}, and keeps the contextualizer unique to the task at hand.

\paragraph{Baselines}
In addition to ProtoMAML, MAML and Prototypical Networks, we also compare our model to two baselines each of which contains one aspect of the improvements we implement in ProtoContext, but neither contains both. The first is a \emph{PCA baseline}, 
which performs PCA on the prototypes and only keeps the principal components with singular values greater than 30\% of the highest singular value (threshold found through grid search).
Prototypes are replaced with the sum of their projections onto these components and the head is then intialized through Equation~\eqref{eq:headinit}. This produces prototypes more specific to the task at hand, thresholding out noise, and is used as 
a comparison to our model's process of producing task-specific prototypes before initialization.

We also test a \emph{Context Only} model which uses our contextualization scheme but does \emph{not} re-initialize the head at every task, instead updating it in the same manner as MAML. This is meant to evaluate the importance of the task-specific feature-adaptation of our contextualizer.


\paragraph{Varying inner loop steps}

Although our primary focus is on task adaptation with a single inner loop step, we also test our models with more inner loop steps to see how they make use of the additional gradient updates.
This is a useful domain to explore, as the fine-tuning time scales linearly with the number of inner loop steps.
When fewer gradient steps are allowed, it is likely that the contextualizer is more involved in the classification of query samples, as the model weights are less adapted to the task. With more inner loop steps, there is risk of overfitting to $\mathbf{S}$, and we seek to understand if our ProtoContext model suffers from this issue.
We experimented with as many as 10 inner loop steps on all datasets, while all of our main results are with 1 step.


\paragraph{Intra-class similarity}
The success of any method that relies on class representation, whether prototypes or contextualizers, is dependent on the ability of $f_{\bd{\theta}}(\cdot)$ to produce representations for each class that align closely to the samples of that class in the latent space. If this is not the case, then it will be difficult to classify new samples based off of the prototype alone. To measure the quality of this clustering, we propose a new metric as follows
\begin{align*}
    \text{intra-class similarity}(a) \colonequals \frac{1}{k} \sum_{j=1}^k \cos^{-1}\left(\frac{f_{\bd{\theta}}\left(\mathbf{x}_{j}^{(a)}\right)}{\left\|f_{\bd{\theta}}\left(\mathbf{x}_{j}^{(a)}\right)\right\|_2}\cdot \frac{\mathbf{p}^{(a)}}{\left\|\mathbf{p}^{(a)}\right\|_2}\right).
\end{align*}

We use the measure of angular similarity rather that Euclidean distance as in \citep{snell-etal-2017-prototypical} because logits are calculated through dot product with a head initialized to the prototypes, not through distance in Euclidean space. We measure the intra-class similarity for the features produced by our ProtoContext model and ProtoMAML before fine-tuning on $\mathbf{S}$. For fair comparison, we only measure the features produced by $f_{\bd{\theta}}$ and do NOT measure the features modified by the context in ProtoContext. Evaluating this measure allows us to understand if the addition of a contextualizer produces feature extractors that create stronger latent representations.

\begin{table}[t]
  \small
  \caption{Accuracy for 1- and 5-shot experiments. Omniglot is 20-way and the rest are 5-way. Median and standard deviation from three trials (five trials for experiments with large deviation). Below the penultimate lines for each experiment are our best models: ProtoContext with Contextualized Prototypes and Head.}
  \centering
  \begin{tabular}{lllll}
    \toprule
    \multicolumn{5}{c}{Accuracy for 1-shot (\%)} \\
    \midrule
    \multicolumn{1}{c}{Model}
    & \multicolumn{1}{c}{Omniglot} 
    & \multicolumn{1}{c}{Mini-ImageNet} & \multicolumn{1}{c}{Tiered-ImageNet}
    & \multicolumn{1}{c}{Airplanes}  \\
    \cmidrule(r){1-1}
    \cmidrule(r){2-2}
    \cmidrule(r){3-3}
    \cmidrule(r){4-4}
    \cmidrule(r){5-5} 
    Prototypical Networks & $80.9\pm 4.00$ & $49.9\pm 0.25$ & $50.3\pm 0.38$ & $42.8\pm 2.76$ \\
    MAML++ & $75.3\pm 2.19$ & $48.5\pm 2.95 $ & $49.2\pm 0.66$ & $39.8\pm 8.50$ \\
    ProtoMAML++ & $70.2 \pm 5.08$ & $47.8\pm 0.25$ & $49.7\pm 0.61$ & $53.1\pm 0.96$ \\
    PCA baseline & $69.4 \pm 4.27$ & $48.5 \pm 0.30$ & $50.0 \pm 0.38$ & $39.5 \pm 1.54$ \\
    Context Only (\emph{Contex. Prototypes}) & $94.9 \pm 0.65$ & $45.4 \pm 2.19$ & $47.4 \pm 5.05$ & $41.3 \pm 2.56$ \\
    Context Only (\emph{Head}) & $94.9 \pm 0.12$ & $49.1 \pm 0.77$ & $50.3 \pm 0.36$ & $52.6 \pm 0.15$ \\
    \midrule
    ProtoContext (\emph{Contex. Prototypes}) & $96.0 \pm 0.40$ & 
    $\textbf{50.1} \pm 1.13$ & 
    $52.1 \pm 1.07$ & $\textbf{58.3} \pm 0.14$ \\
    ProtoContext (\emph{Head}) & $\textbf{96.7} \pm 0.22$ & $\textbf{50.1} \pm 0.21$ & $\textbf{54.3} \pm 0.32$ & 
    $56.6 \pm 0.19$ \\
    \midrule
    \multicolumn{5}{c}{Accuracy for 5-shot (\%)} \\
    \midrule
    \multicolumn{1}{c}{Model}
    & \multicolumn{1}{c}{Omniglot} 
    & \multicolumn{1}{c}{Mini-ImageNet} & \multicolumn{1}{c}{Tiered-ImageNet}
    & \multicolumn{1}{c}{Airplanes}  \\
    \cmidrule(r){1-1}
    \cmidrule(r){2-2}
    \cmidrule(r){3-3}
    \cmidrule(r){4-4}
    \cmidrule(r){5-5}
    Prototypical Networks & $91.5\pm 2.89$ & $\textbf{66.4}\pm 0.40$ & $70.1\pm 0.40$ & $63.0\pm 0.86$ \\
    MAML++ & $84.6\pm 2.40$ & $65.6\pm 0.39$ & $57.3\pm 6.87$ & $58.2\pm 9.73$ \\
    ProtoMAML++ & $77.5\pm 4.78$ & $60.7\pm 0.99$ & $65.0\pm 0.31$ & $63.3\pm 0.57$ \\
    PCA baseline & $79.1 \pm 4.06$ & $61.5 \pm 0.13$ & $65.8 \pm 0.25$ & $62.2 \pm 2.06$ \\
     Context Only (\emph{Context. Prototypes}) & $98.1 \pm 0.22$ & $60.0 \pm 0.30$ & $59.8 \pm 0.21$ & $62.7 \pm 4.39$ \\
    Context Only \emph{(Head)} & $98.3 \pm 0.08$ & $63.0 \pm 0.12$ & $66.8 \pm 0.28$ & $63.0 \pm 3.07$ \\
    \midrule
     ProtoContext (\emph{Contex. Prototypes}) & $\textbf{98.4} \pm 0.15$ & $63.9 \pm 0.17$ & $70.0 \pm 0.19$ & $\textbf{70.1} \pm 0.61$ \\
    ProtoContext (\emph{Head}) & $\textbf{98.4} \pm 0.04$ & $64.5 \pm 0.39$ & $\textbf{71.4} \pm 0.32$ & $69.3 \pm 0.32$ \\
    \bottomrule
  \end{tabular}
  \label{1shot}
\end{table}


\section{Results}


\paragraph{1-shot and 5-shot classification}
In Table~\ref{1shot} we present our results on 1-shot and 5-shot classifications. In 1-shot classification, we outperform the baselines 
across all datasets. The 1-shot setting is challenging for gradient-based meta learners as there are not many data samples to use for adaptation and overfitting is a high risk. It becomes very useful in such a setting to have a method of augmenting prediction such as the one we introduce through contextualization. 
In the 5-shot setting, we outperform on Omniglot, Tiered-ImageNet, and Airplanes. On Mini-ImageNet all gradient based models are outperformed by Prototypical Networks. 
In order to understand the limited performance of our model on this dataset, we present figures of the additional inner loop steps in the SM which suggest that ProtoContext is unable to develop a sufficiently strong initialization in this setting. We should note that our models outperform ProtoMAML and the PCA baseline, which also rely on the head initialization.

The two forms of contextualization we explore perform roughly similarly on Omniglot and Mini-ImageNet. On Tiered-ImageNet, using the head as contextualization leads to considerable benefits, while on Airplanes, using contextualized prototypes produces better results. 
We also see that in general, the Context Only baseline is more successful when given the head as contextualization. This makes sense, since although there is no explicit initialization of the head in this model, contextualization allows us to augment features with information about the head. Thus, features of samples in a given class can be made to align strongly with the weights in the head for that class. A surprising benefit of our contextualization method is the low variance across runs. Many methods have large standard deviations in the 1 and 5 shot settings (up to $9.73\%$ for MAML). In contrast, across all datasets our ProtoContext model has low variance in its test accuracy. 

\begin{figure}[t]
\centering

\begin{minipage}[b]{0.45\linewidth}
\includegraphics[scale=0.40]{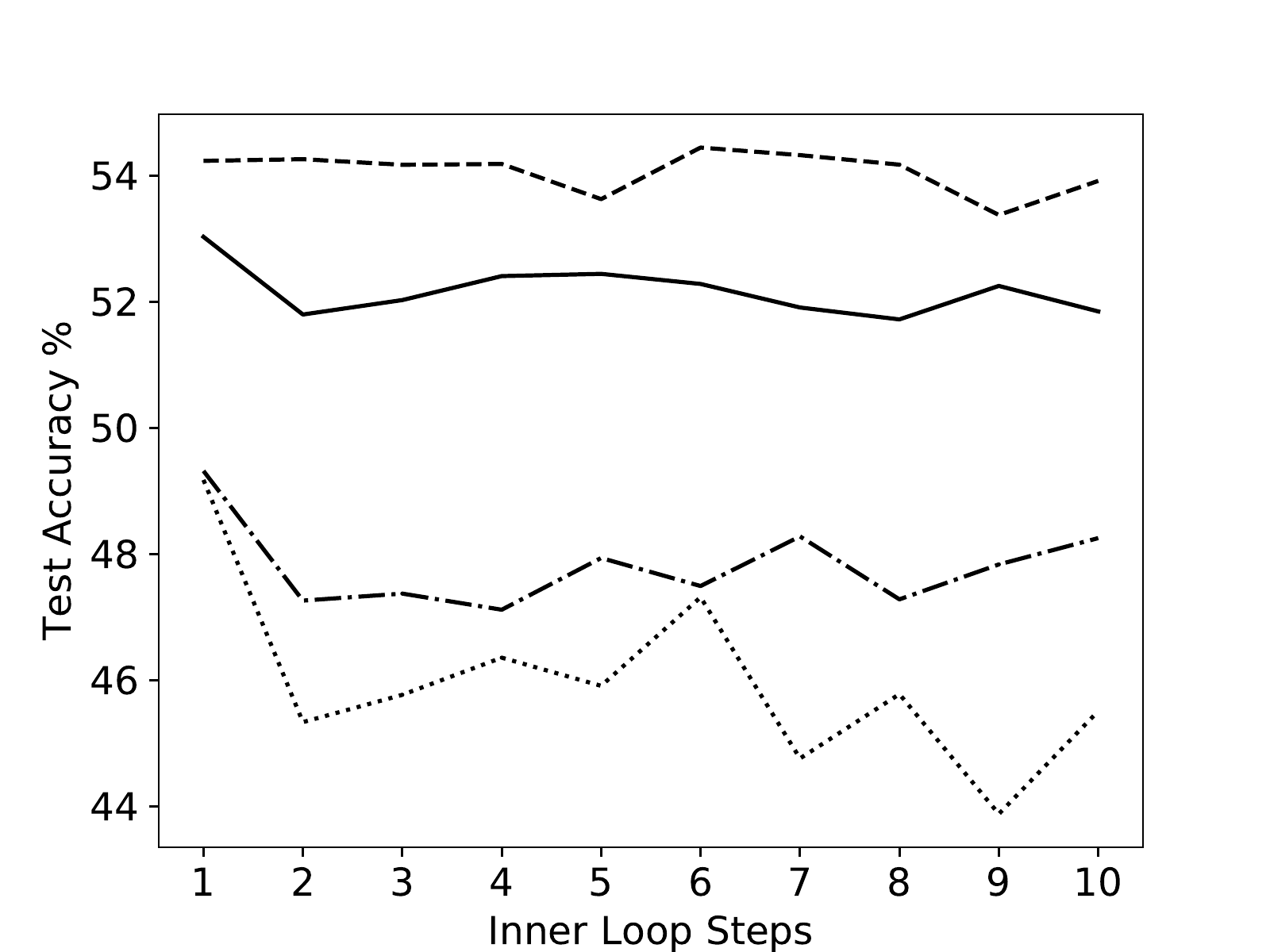}

\end{minipage}
\quad
\begin{minipage}[b]{0.40\linewidth}
  \includegraphics[scale=0.40]{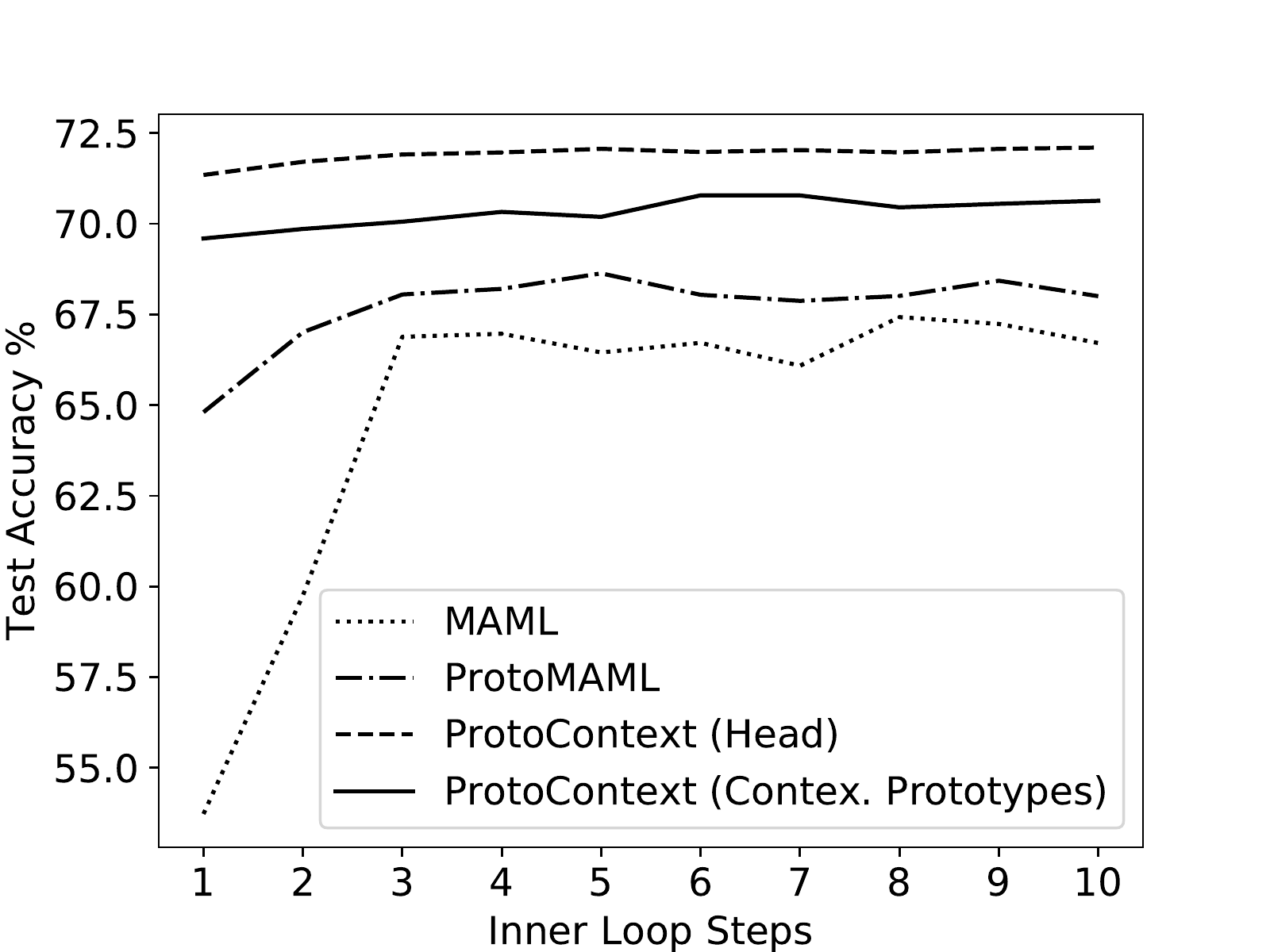}

\end{minipage}
\caption{Test Accuracy versus the number of inner loop steps for 1-shot (left) and 5-shot (right) on Tiered-ImageNet. ProtoContext shows strong results with just one adaptation step. In addition, ProtoContext seems to suffer from significantly less overfitting than ProtoMAML or MAML in the 1-shot setting.}
\label{InnerLoop}

\end{figure}

\paragraph{Inner loop steps}
In Figure~\ref{InnerLoop} we demonstrate our results from varying the number of inner loop steps the model is allowed before testing on the query set on Tiered-ImageNet. We see that in both the 1-shot and the 5-shot settings, our model both begins at and stays at a high level of accuracy. These results indicate that ProtoContext is able to effectively develop a very strong initialization for the head that requires little task-specific fine-tuning. In the 1-shot setting the accuracy of MAML and ProtoMAML decreases as we add more inner loop steps, which may be due to overfitting. While both versions of ProtoContext do suffer from this to a degree, the losses in accuracy are not nearly as dramatic. This suggests that the use of contextualizers may provide a method of augmenting prediction that is not as prone to overfitting as fine-tuning parameters through gradient descent, or that the initializations created by ProtoContext produce heads that are more general than those produced by ProtoMAML and MAML. Furthermore, a practical benefit of relying on a single step of inner loop adaptation is that it allows for a lower training budget for few shot learning.

\section{Analysis}
\label{sec:analysis}

   
    

In this section our goal is to understand the role of contextualizers in conjuction with (\emph{i}) the feature extactor and (\emph{ii}) the self-attention mechanism. Hence, we develop two approaches to observe the role of contextualizers, which we present below.

\paragraph{Intra-class similarity}
Table~\ref{IntraClassSim} shows the results of measuring the intra-class similarity of the features extracted by our ProtoContext model as well as ProtoMAML. Across every dataset, we see that ProtoContext consistently produces features with a lower intra-class similarity, indicating that it has learned a feature extractor that is discriminative. These numbers suggest that the use of our contextualization mechanism may have some upstream benefit to the feature extractor. We can understand this by realizing that, because our contextualizers are initialized with versions of the prototypes, properly making use of the contextualization mechanism requires contextualizers that well represent the samples given. 

\begin{table}[b]
  \small
  \caption{Intra-class similarity is noticeably smaller for ProtoContext than ProtoMAML. Experiments are 5-shot. All measures in radians. Mean and standard deviation from three trials.}
  \centering
  \begin{tabular}{lccc}
    \toprule
    \multicolumn{1}{c}{Dataset}
    & \multicolumn{1}{c}{ProtoContext  (\emph{Contex. Prototypes})}
    & \multicolumn{1}{c}{ProtoContext (\emph{Head})} 
    & \multicolumn{1}{c}{ProtoMAML}  \\
    \cmidrule(r){1-1}
    \cmidrule(r){2-2}
    \cmidrule(r){3-3}
    \cmidrule(r){4-4}
   
    Omniglot & $0.37\pm 0.006$ & $0.28 \pm 0.001$ & $0.84 \pm 0.011$  \\
    Mini-ImageNet &  $0.99 \pm 0.001$ & $0.82 \pm 0.003$ & $1.24 \pm 0.006$  \\
    Tiered-ImageNet & $0.98 \pm 0.007$ & $0.83 \pm 0.002$ & $1.22 \pm 0.003$  \\
    Airplanes & $0.93 \pm 0.128$ & $0.70 \pm 0.005$ & $1.15 \pm 0.019$  \\
    
    \bottomrule
  \end{tabular}
  \label{IntraClassSim}
\end{table}

\paragraph{Use of contextualizers in self-attention}
Figure~\ref{attn weights} shows heatmaps of the attention weights in the self-attention mechanism of our ProtoContext model. The diagonal patterns of attention indicate high weight is given to the contextualizers that correspond to the class of a given sample. The context is also highly updated by the samples in its class, although this is more prominent in the 1-shot than the 5-shot setting. We present additional heatmaps in the SM that show a range of patterns can emerge, and show how applying regularization on the heatmaps can lead to improvements.  We also find that the query almost never attends the query during self-attention. This is especially important as this is a fully learned behavior, not one we code into the model. We believe this demonstrates that the self-attention mechanism is learning what we expect: that context is a much more useful tool than other samples for making classification decisions.

\begin{figure}[t]
    \centering
    
\begin{minipage}[t]{0.45\linewidth}
\includegraphics[scale=0.25]{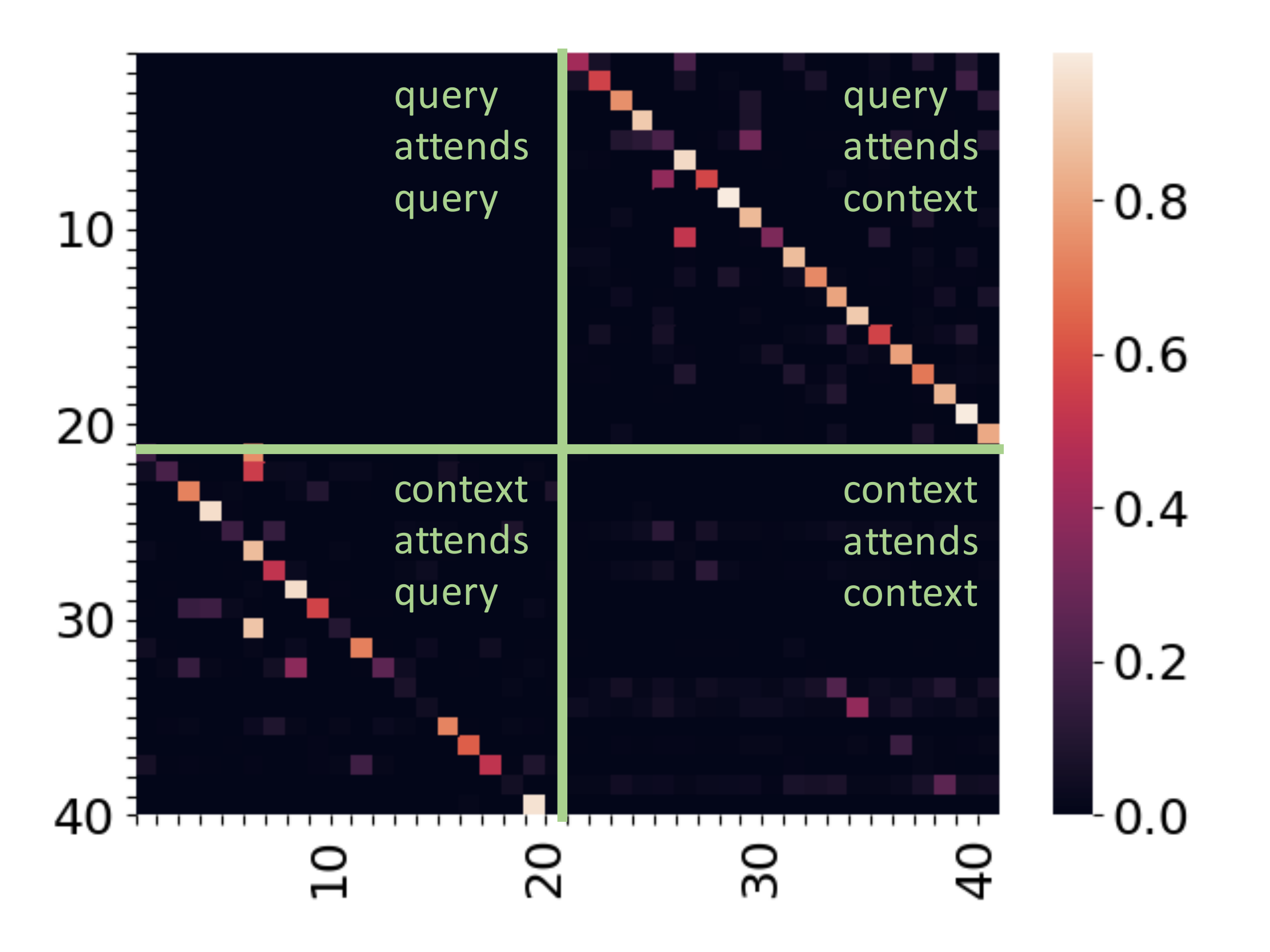}

\end{minipage}
\quad
\begin{minipage}[t]{0.45\linewidth}
  \includegraphics[scale=0.25]{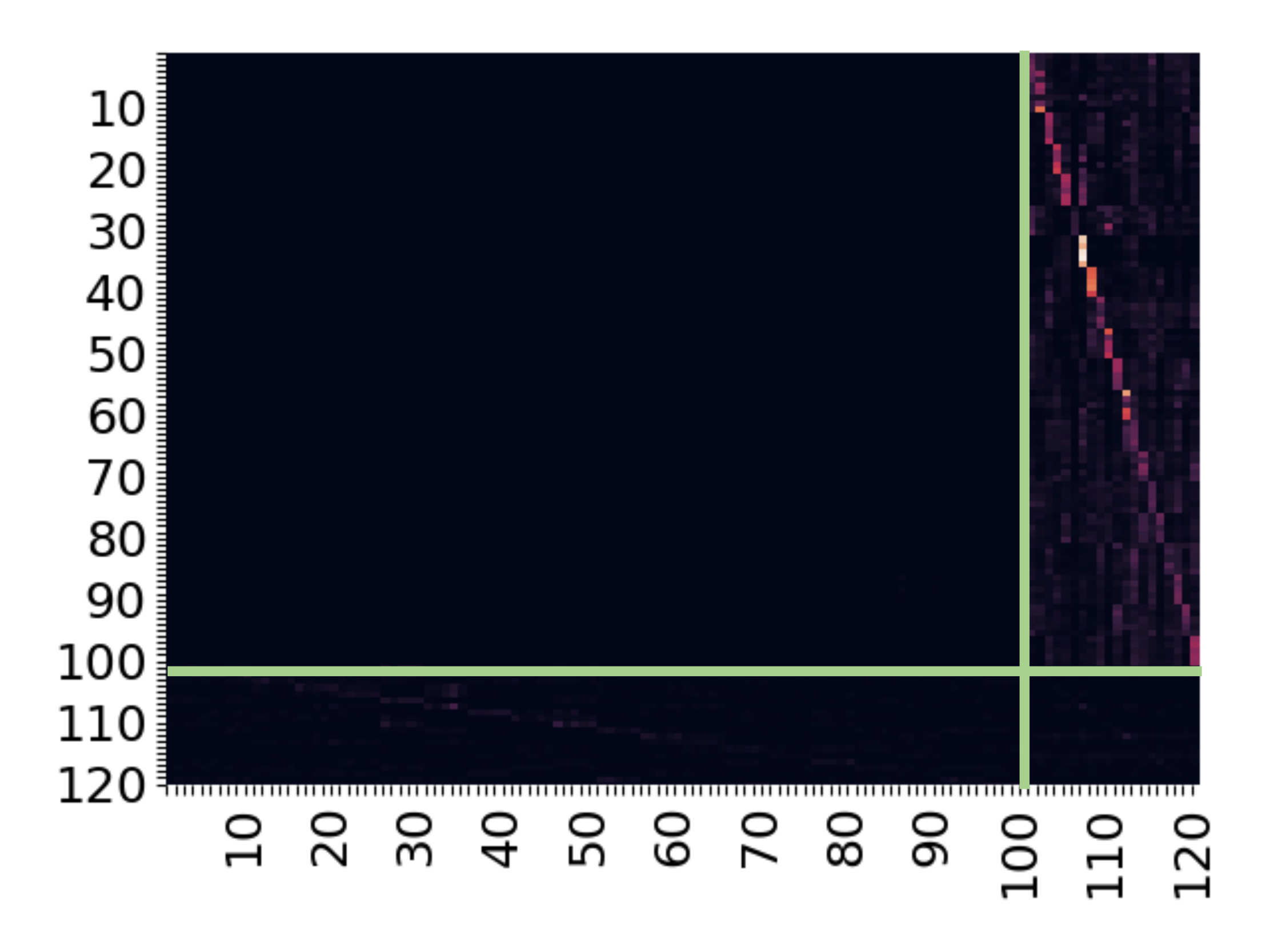}
\end{minipage}
\label{attn weights}
\caption{Heatmaps of the attention weights for Omniglot provide evidence to support our conjecture in Equation~\eqref{eq:expansion}: significant weight is given to the contextualizer corresponding to each example's class when contextualizing the features. This is seen in the strong diagonals in the heatmaps, which are a result of samples being presented in the order of class number.} 
\end{figure}

\section{Related Work}


\paragraph{Few shot learning}
Few shot learning is, broadly, the task of making inference with use of few labeled examples \citep{Li_2019, Lifchitz_2019}. To solve this fairly formidable problem, many fields have been explored including meta learning \citep{Sun_2019, li2017metasgd, DBLP:journals/corr/abs-1803-00676, finn2017}, metric learning \citep{Snell2017PrototypicalNF, vinyals2016, cheng2019fewshot, Karlinsky_2019_CVPR}, as well as more broad approaches \citep{ravi2017, InfoRetreivalTriantafillou2017, garcia2017fewshot, Lee_2019_CVPR}. Closest to our work are methods that approach few shot learning through forms of contextualization \citep{Changpinyo_2017, oreshkin2018tadam, Gidaris_2018}. We were particularly encouraged by the results of ~\citet{ye2018fewshot}, who found that attention mechanisms are a very powerful set-to-set function for few shot learning.  However, no approach that we are aware of makes use of attention not only to construct task-specific initializations, but also to modify features in a meta-learning model.

\paragraph{Meta learning}
Meta learning algorithms create models that quickly adapt to new tasks \citep{yin2019metalearning}. The goals of meta learning fit those of few shot learning very well, and many meta learning ideas have been successfully applied to few shot learning \citep{Jamal_2019_CVPR, BayseianMAML, gordon2018metalearning}. Approaches vary, but can broadly be split into those that focus on using architectures or learned gradient updates \citep{santoro2016, xu2018metagradient, wichrowska2017learned, kordik2010meta, flennerhag2019metalearning} and those that make use of learned initializations which can quickly adapt to new tasks through gradient descent \citep{finn2017, nichol2018firstorder, riemer2018learning, MetaInitDauphin2019, finn2017metalearning, frans2017meta}. The latter approach is much closer to our work. We take this approach one step further and construct \emph{task-specific} initializations before fine-tuning occurs along the lines of~ \citet{triantafillou2019}. In addition, we produce a task-specific feature space through our contextualizers that is somewhat similiar to the ideas presented in~\citep{zintgraf2018fast, perez2017film}. Our mechanism is very different, however, using self-attention to learn our feature space. Furthermore, none of these approaches explore both empirically and theoretically the benefits provided by contextualization in very challenging environments, such as allowing only 1 gradient step for adaptation.

\paragraph{Metric learning}
Deep metric learning produces algorithms and models which are able to construct ``metrics'' by which images can be compared or retrieved. Much work focuses on loss functions in metric learning \citep{Ge_2018, Song_2016, sohn2016improved, movshovitz2017no, Wang_2017}, and some of this work has led to methods useful for few shot learning \citep{roweis2004neighbourhood, snell-etal-2017-prototypical, salakhutdinov2007learning, snell-etal-2017-prototypical}. Metric learning has also been used in one-shot learning \citep{koch2015siamese, vinyals2016}. This work is useful in comparison to our gradient-based methods; however, we seek to extend it by using strong metrics for initialization of gradient learners (i.e. Equation~\eqref{eq:headinit}).

\section{Conclusion and Future Work}
We have presented a novel framework---Contextualization---for approaching few shot classification which allows for task-specific initialization and feature modification. We tested two forms of contextualization and showed that both outperform strong baselines. In addition, we presented other benefits of contextualization including resilience to overfitting, potential upstream benefits for feature extractors, and use of context to boost classification decisions. We believe that with proper contextualizers, contextualization can be extended to other tasks such as reinforcement learning and regression, and we hope to explore this in future work.

\section*{Acknowledgements}

We would like to thank Peter Lu, Charlotte Loh, Ileana Rugina, Kristian Georgiev, Brian Chuang, Alireza Fallah, Séb Arnold, and Chelsea Finn for fruitful conversations.

Research was sponsored in part by the United States Air Force Research Laboratory and was accomplished under Cooperative Agreement Number FA8750-19-2-1000. This material is also based upon work supported in part by the U. S. Army Research Office through the Institute for Soldier Nanotechnologies at MIT, under Collaborative Agreement Number W911NF-18-2-0048. The views and conclusions contained in this document are those of the authors and should not be interpreted as representing the official policies, either expressed or implied, of the United States Air Force or the U.S. Government. The U.S. Government is authorized to reproduce and distribute reprints for Government purposes notwithstanding any copyright notation herein.  

\bibliography{neurips_2020_refs}
\bibliographystyle{unsrtnat}

\newpage
\appendix
\begin{center}
\textbf{\huge Supplementary Materials}
\end{center}

We organize the Supplemental Materials into three categories: {(\emph{i}) Formalism and Theory}; (\emph{ii}) Empirical Analysis and (\emph{iii}) Experiments. 

\emph{Formalism and Theory}: In Section~\ref{sec:details} we present details of our contextualization algorithm. In Section~\ref{sec:updates} we compute explicit forms of gradient updates under contextualization and analyze the expressions. 

\emph{Empirical Analysis}: In Section \ref{sec:heatmaps} we present a method of constraining our attention mechanism to better focus on the contextualizers. 

\emph{Experiments}: In Section~\ref{sec:datasets} we describe the datasets used for experiments. In Section~\ref{sec:xavier} we present a modification of the prototypical initialization used in ProtoMAML, which we introduced in order to secure stable training. In Section~\ref{sec:training} we discuss training specifications. In Section~\ref{sec:quickdraw} we present additional results on the Quickdraw dataset. In Section~\ref{sec:innerloop} we present plots of additional inner loop steps for more datasets.  

\section{Details of the Contextualization Algorithm}
\label{sec:details}

\begin{figure*}
  \centering
  \includegraphics[width=\textwidth]{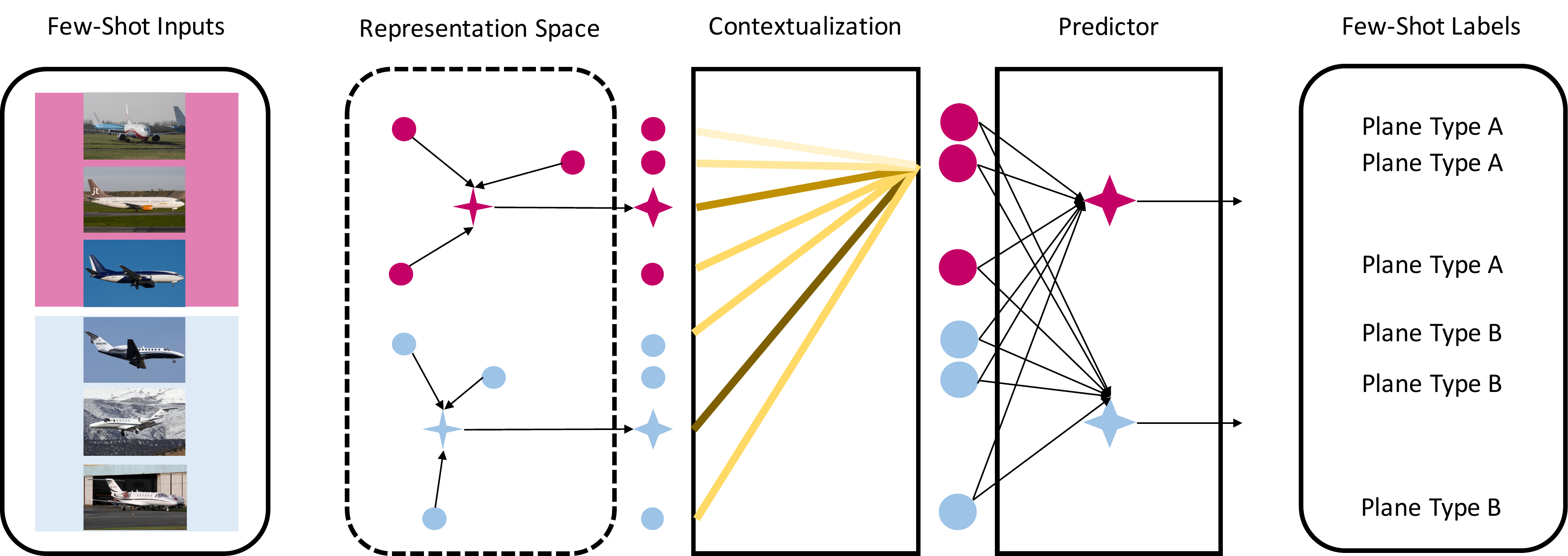}
  \caption{Rough sketch of our ProtoContext method. Increased thickness indicates updated representations. Horizontal lines in the Representation Space indicate a parametrized transformation of the prototypes. Yellow color indicates the strength of the attention connections for a particular representation. The predictor compares the contextualized input representations to the contextualized prototypes to make a prediction. Bias in the predictor is ignored for simplicity.}
  \label{fig:sketch}
\end{figure*}

In this section, we describe in greater detail the algorithm by which our model implements contextualization in the inner loop. Figure~\ref{fig:sketch} serves as an intuitive sketch of ProtoContext working on a few-shot task with two classes, Plane Type A and B. 

Below we proceed with the formalization of our algorithm. We note that $(\mathbf{X}, \mathbf{y})$ is the support set, $(\mathbf{X'}, \mathbf{y'})$ is the query set, and $\mathbf{C}$ are our contextualizers.
\paragraph{Contextualization algorithm}
The inner loop, specified by the weights $(\bd{\theta},\bd{\phi},\bd{\psi})$, is defined recursively for $i=1,\dots,M$ (where $M$ is the number of inner loop steps) by setting 
 $\left(\bd{\theta}^{(0)},\bd{\phi}^{(0)},\bd{\psi}^{(0)}\right) \colonequals \left(\bd{\theta},\bd{\phi},\bd{\psi} \right)$ and
the forward pass as follows
\begin{align*}
\widetilde{\mathbf{X}}^{(i)},\mathbf{C}^{(i)} & = s_{\bd{\phi}^{(i-1)}}\left( f_{\bd{\theta}^{(i-1)}}(\mathbf{X}), \mathbf{C}^{(i-1)}\right) & 
\hat{\mathbf{y}}^{(i)} & = p_{\bd{\psi}^{(i-1)}}\left(\widetilde{\mathbf{X}}^{(i)}\right)
\end{align*}
and a gradient step with step size $\alpha$ yields $\left(\bd{\theta}^{(i)},\bd{\phi}^{(i)},\bd{\psi}^{(i)}\right)$ respectively expressed as 
\begin{align*}
 \bd{\theta}^{(i)} & = \bd{\theta}^{(i-1)} - \alpha \cdot \nabla_{\bd{\theta}^{(i-1)}} L\left(\hat{\mathbf{y}}^{(i)},\mathbf{y}\right) \\ \bd{\phi}^{(i)} & = \bd{\phi}^{(i-1)} - \alpha \cdot \nabla_{\bd{\phi}^{(i-1)}} L\left(\hat{\mathbf{y}}^{(i)},\mathbf{y}\right) \\
 \bd{\psi}^{(i)} & = \bd{\psi}^{(i-1)} - \alpha \cdot \nabla_{\bd{\psi}^{(i-1)}} L\left(\hat{\mathbf{y}}^{(i)},\mathbf{y}\right), 
\end{align*}
where for simplicity we omit the dependence of $\hat{\mathbf{y}}^{(i)}$ on $(\bd{\theta},\bd{\phi},\bd{\psi})$. Now, our forward pass on $\mathbf{Q}$ is
\begin{align*}
\widetilde{\mathbf{X}}, \mathbf{C} & = s_{\bd{\phi}^{(M)}}\left( f_{\bd{\theta}^{(M)}}\left(\mathbf{X}'\right), \mathbf{C}^{(M)}\right) & 
\hat{\mathbf{y}} & = p_{\bd{\psi}^{(M)}}\left(\widetilde{\mathbf{X}}\right)
\end{align*}
and our weights are updated in the outer loop using normal gradient descent:
\begin{align*}
\bd{\theta} & = \bd{\theta} - \alpha \cdot \nabla_{\bd{\theta}^{(M)}} L\left(\hat{\mathbf{y}},\mathbf{y'}\right) & 
\bd{\phi} & = \bd{\phi} - \alpha \cdot \nabla_{\bd{\phi}^{(M)}} L\left(\hat{\mathbf{y}},\mathbf{y'}\right) &
\bd{\psi} & = \bd{\psi} - \alpha \cdot \nabla_{\bd{\psi}^{(M)}} L\left(\hat{\mathbf{y}},\mathbf{y'}\right). 
\end{align*}
We make special note of the fact that \textbf{no gradient passes through the contextualizers}, as this not only greatly slows down computation of the gradient, but also can cause exploding gradients with multiple sets. We also note that in the case of the Head ProtoContext model, each $\mathbf{C}^{(i)}$ is replaced with the Head of the model at that step.

\section{Gradient Updates}
\label{sec:updates}
In this section, we show through analysis of our gradient the impact the addition of a self attention mechanism can have on our parameter updates. Our analysis is focused on the 1-shot setting for the sake of simplicity of notation. All of the calculations extend to the 5-shot case naturally.
\subsection{Notation for the 1-shot setting}
Note that the inner loop updates depend only on the support set $\mathbf{S}$, and since in the 1-shot setting we have a single example from each class $a$, we can write the support set as follows $\mathbf{S}=\left(\mathbf{X},\mathbf{y}\right)\equiv \left\{ \left(\mathbf{x}^{(a)}, a\right)\right\}_{a=1}^{n}$, where $a=y^{(a)}$ without loss of generality, i.e. the target for the input $\mathbf{x}^{(a)}$ is the index of its class, which is $a$. Likewise, we denote the context as $\mathbf{C} \equiv \left\{ \mathbf{c}^{(a)}\right\}_{a=1}^n.$ Finally, let $\mathcal{C}^{(a)}$ be the contextualization of input example $\mathbf{x}^{(a)},$ i.e.
$\mathcal{C}^{(a)}=s_{\bd{\phi}}\left(f_{\bd{\theta}}\left(\mathbf{x}^{(a)}\right),\mathbf{C}\right)$ for the contextualization algorithm and
$\mathcal{C}^{(a)}=f_{\bd{\theta}}\left(\mathbf{x}^{(a)}\right)$ for any other gradient based algorithm: MAML, ProtoMAML, etc.
\subsection{General form of the loss and the gradient updates}
Our approach is to impose the structure of the classification head in order to obtain an explicit form of the loss function, which will consequentially yield the gradient updates for the parameters of our model. We proceed with our analysis below.

Let the weights of the head be $\left\{\bd{\psi}^{\left(a''\right)}\right\}_{a''=1}^n$ and their corresponding biases be $\left\{b^{\left(a''\right)}\right\}_{a''=1}^n.$ Then, we have that for a single sample
\begin{equation}
\label{eq:output_of_network}
p_{\bd{\psi}}\left(s_{\bd{\phi}}\left(f_{\bd{\theta}}\left(\mathbf{x}^{(a)}\right)\right)\right) = \mathrm{softmax} \left( \begin{bmatrix}
\bd{\psi}^{(1)} \cdot \mathcal{C}^{(a)}  + b^{(1)} \\
\vdots\\
\bd{\psi}^{(n)} \cdot \mathcal{C}^{(a)}+ b^{(n)}
\end{bmatrix} \right).
\end{equation}
Note that in our setting the cross entropy loss takes the form 
\begin{equation}
\label{eq:loss_form}
-\sum_{a''=1}^n \log \left[p_{\bd{\psi}}\left(s_{\bd{\phi}}\left(f_{\bd{\theta}}\left(\mathbf{x}^{\left(a''\right)}\right)\right)\right)\right]_{a''},
\end{equation}
where $[\mathbf{z}]_{a''}$ means that we take the $a''$-th component of the vector $\mathbf{z}$, and $p_{\bd{\psi}}, s_{\bd{\phi}},$ and $f_{\bd{\theta}}$ are the predictor (head), contextualization mechanism, and feature extractor respectively as defined above.  Now, using Equation~\eqref{eq:output_of_network} in
Equation~\eqref{eq:loss_form}, then using the form of the softmax and simplifying the expression we obtain the following loss function $L$ viewed as a function of the support set $\mathbf{S}$ as follows

\begin{equation}
\label{eq:loss}
L(\mathbf{S})=
-\sum_{a''=1}^n \bd{\psi}^{\left(a''\right)}\cdot \mathcal{C}^{\left(a''\right)}   + b^{\left(a''\right)} 
 + \sum_{a''=1}^n \log \left( \sum_{\widetilde{a}=1}^n \exp\left( \bd{\psi}^{\left(\widetilde{a}\right)}\cdot \mathcal{C}^{\left(a''\right)}  + b^{\left(\widetilde{a}\right)} \right) \right). 
\end{equation}

Hence, the form of this loss in Equation~\eqref{eq:loss} is amenable to analysis for the predictor and feature extractor of our model. We proceed in this order below.
\subsection{Gradient Updates for the Predictor}
Differentiating Equation~\eqref{eq:loss} with respect to the head weights $\bd{\psi}^{(a)}$, corresponding to class $a$, we obtain the following
\begin{equation}
\label{eq:headform}
\nabla_{\bd{\psi}^{(a)}}L(\mathbf{S}) = -\mathcal{C}^{(a)} + \mathcal{D}^{(a)},
\end{equation}
where we introduced the following notation
\begin{equation*}
\label{eq:term_D}
\mathcal{D}^{(a)} = 
\sum_{a''=1}^n \frac{\exp\left( \bd{\psi}^{\left(a\right)}\cdot \mathcal{C}^{\left(a''\right)} + b^{\left(a\right)} \right)}{\sum_{\widetilde{a}=1}^n \exp\left( \bd{\psi}^{\left(\widetilde{a}\right)}\cdot \mathcal{C}^{\left(a''\right)} + b^{\left(\widetilde{a}\right)} \right)} \mathcal{C}^{\left(a''\right)} \equiv 
\sum_{a''=1}^{n}\frac{\exp\left( \bd{\psi}^{\left(a\right)}\cdot \mathcal{C}^{\left(a''\right)} + b^{\left(a\right)}\right)}{\mathcal{Z}}\mathcal{C}^{\left(a''\right)},
\end{equation*}
where $\mathcal{Z}$ is the partition function. Equation~\eqref{eq:headform} is significant since it tells us that the contextualization algorithm can control the gradient updates  
through the self-attention mechanism, because the contextualization $\mathcal{C}^{(a)}$ depends on the parameters $\bd{\phi}$ of the self-attention mechanism.
Under some assumptions, this behavior might yield simplifications, amenable to analysis. In that spirit, we proceed with the following
\begin{prop}
Assume our contextualizations are orthogonal. In the 1-shot setting, the inner loop updates for the weights in the head for class $a$ move in a direction of \textbf{positive} correlation with the contextualization of its support example $\mathbf{x}^{(a)}.$ 
\end{prop}
\begin{proof}
Note that 
$0< \exp\left( \bd{\psi}^{\left(a\right)}\cdot \mathcal{C}^{\left(a\right)} + b^{\left(a\right)} \right)/\mathcal{Z} < 1$. Combined with our assumption of orthogonality of contextualizations, we then get that the correlation with the gradient update (ignoring the learning rate) is given as follows
\begin{align*}
    \mathcal{C}^{\left(a\right)} \cdot \left(-\nabla_{\bd{\psi}^{(a)}}L(\mathbf{S})\right) &=  \mathcal{C}^{\left(a\right)} \cdot \left(\mathcal{C}^{(a)} - \mathcal{D}^{(a)}\right)\\
    &= \mathcal{C}^{\left(a\right)} \cdot \mathcal{C}^{\left(a\right)} - \mathcal{C}^{\left(a\right)} \cdot \mathcal{D}^{(a)} \\
    &= \left\|\mathcal{C}^{\left(a\right)}\right\|^2_2 - \mathcal{C}^{\left(a\right)} \cdot \mathcal{D}^{(a)}\\
    &= \left\|\mathcal{C}^{\left(a\right)}\right\|^2_2 - \frac{\exp\left( \bd{\psi}^{\left(a\right)}\cdot \mathcal{C}^{\left(a\right)} + b^{\left(a\right)} \right)}{\sum_{\widetilde{a}=1}^n \exp\left( \bd{\psi}^{\left(\widetilde{a}\right)}\cdot \mathcal{C}^{\left(a\right)} + b^{\left(\widetilde{a}\right)} \right)} \left\|\mathcal{C}^{\left(a\right)}\right\|^2_2 \\
    &= \left(1 - \frac{\exp\left( \bd{\psi}^{\left(a\right)}\cdot \mathcal{C}^{\left(a\right)} + b^{\left(a\right)} \right)}{\sum_{\widetilde{a}=1}^n \exp\left( \bd{\psi}^{\left(\widetilde{a}\right)}\cdot \mathcal{C}^{\left(a\right)} + b^{\left(\widetilde{a}\right)} \right)}\right) \left\|\mathcal{C}^{\left(a\right)}\right\|^2_2 \\
    &> 0,
\end{align*}
where going from the third to the fourth line we use the fact that $\mathcal{C}^{\left(a\right)} \cdot \mathcal{C}^{\left(a'\right)}=0$ if and only if $a$ is different from $a'$. From here, since the learning rate scales each line above by $\alpha > 0$, the proof follows, as desired. 
\end{proof}
\subsection{Gradient Updates for the Feature Extractor}
For this analysis we would like to underline the dependence of the contextualizations $\mathcal{C}^{\left(a''\right)}$ on the parameters $\bd{\theta}$ of our feature extractor, by explicitly writing the dependence as follows $\mathcal{C}_{\bd{\theta}}^{\left(a''\right)}.$ Hence, differentiating Equation~\eqref{eq:loss} with respect to $\bd{\theta}$ and using notation from the previous section we obtain the following expression for the gradient update
\begin{equation*}
\nabla_{\bd{\theta}}L(\mathbf{S}) = -\sum_{a''=1}^{n}\nabla_{\bd{\theta}}\mathcal{C}_{\bd{\theta}}^{\left(a''\right)} \bd{\psi}^{\left(a''\right)}+\sum_{a''=1}^{n}\sum_{a'=1}^{n}\frac{\exp\left( \bd{\psi}^{\left(a'\right)}\cdot \mathcal{C}_{\bd{\theta}}^{\left(a''\right)} + b^{\left(a\right)} \right)}{\mathcal{Z}}\nabla_{\bd{\theta}}\mathcal{C}_{\bd{\theta}}^{\left(a''\right)}\bd{\psi}^{\left(a'\right)}.
\end{equation*}
We proceed with the following
\begin{prop}
With the contextualization algorithm, the gradient updates for the feature extractor share gradient information from each example in the support set, as opposed to MAML and ProtoMAML.
\end{prop}
\begin{proof}
It suffices to analyse $\nabla_{\bd{\theta}}\mathcal{C}_{\bd{\theta}}^{\left(a''\right)}$
in the above equation. Our self-attention mechanism consists of a scaled dot product attention which yields a linear combination across the transformed (by value weights $\mathbf{W}_{\mathrm{value}}$)
inputs  to the mechanism, followed by a
layer normalization with mean $\mu$ and standard deivation $\sigma$, and then a linear layer with weights $\mathbf{W}$ and bias $\mathbf{b}$ (see section~\ref{sec:training}). Therefore, since the initalization of our context consists of self-attention over the extracted features of the support set, without loss of generality we have that
\[\mathcal{C}_{\bd{\theta}}^{\left(a''\right)} = \mathbf{W} \left( \sum_{\gamma=1}^{n}
v_\gamma^{\left(a''\right)}\frac{\mathbf{W}_{\mathrm{value}}f_{\bd{\theta}}\left(\mathbf{x}^{\left(\gamma \right)}\right)-\mu\bd{1}}{\sigma}\right)+\mathbf{b},
\]
where the coefficients $v_\gamma^{\left(a''\right)} \equiv v_\gamma^{\left(a''\right)}(\bd{\theta},\bd{\phi})$ depend both on the feature extractor and the key and query matrices.
Now, after taking the gradient with respect to $\bd{\theta}$ we obtain the following expression 
\begin{equation*}
\label{eq:featureform}
\nabla_{\bd{\theta}}\mathcal{C}_{\bd{\theta}}^{\left(a''\right)}=\mathbf{W} \left( \sum_{\gamma=1}^{n}
\nabla_{\bd{\theta}}v_\gamma^{\left(a''\right)}\frac{\mathbf{W}_{\mathrm{value}}f_{\bd{\theta}}\left(\mathbf{x}^{\left(\gamma \right)}\right)-\mu\bd{1}}{\sigma}+\boxed{\sum_{\gamma=1}^{n}\frac{v_\gamma^{\left(a''\right)}}{\sigma}\mathbf{W}_{\mathrm{value}}\nabla_{\bd{\theta}}f_{\bd{\theta}}\left(\mathbf{x}^{\left(\gamma
\right)}\right)}\; \right),
\end{equation*}
where we have boxed the contribution from the gradient information coming from \emph{all} support inputs. Contrast this with 
$\mathcal{C}^{\left(a''\right)}=f_{\bd{\theta}}\left(\mathbf{x}^{\left(a''\right)}\right)$ for MAML and ProtoMAML, which yields gradient information only for the support example $a''.$ Thus, the statement follows.
\end{proof}
This proposition is significant since it can explain why the feature extractor yields better intra-class similarity, as we presented in the main text. We conjecture that this is true since gradient information flows from all support exmples in a controlled manner, manifested by the coefficients $v_\gamma^{\left(a''\right)}/\sigma$, which are learned by the self-attention mechanism.


\subsection{Results in the Broader Context of Meta Learning}
We should note that in this section we have described how the self-attention controls the gradient updates during fine-tuning. This emphasizes the role of the self-attention as a meta learner in a similar fashion to how LSTMs can be used as meta learners~\citep{ravi2017}. 
\section{Attention Loss and Heatmaps}
\label{sec:heatmaps}
As noted in the main paper, in some datasets ProtoContext's attention mechanism focuses almost all the attention for a given sample on the contextualizer for its class. However, this is not the case for all datasets. To try and remedy this, we introduce an ``Attention Loss'' which is a regularization term we append to our loss during both inner and outer loop training. The new loss for each sample is defined as follows:
\begin{equation*}
    L_{\mathrm{attention}} = L_{\mathrm{CE}} - \sum_{i=1}^n 1^i_{c}\log(\bd{\Xi}_i) + \left(1- 1^i_{c}\right)\log(1-\bd{\Xi}_i)
\end{equation*}
Where $L_{\mathrm{CE}}$ is the normal cross entropy loss, $\bd{\Xi}_i$ is the attention weight given to the contextualizer for class $i$ for this sample and $1^i_{c}$ is an indicator random variable that is 1 if and only if our sample is of class $i$. We find that delaying addition of this loss until 20 epochs have passed works best, as it allows the self-attention mechanism to learn on more well-defined features. In Figure~\ref{fig:aircraftHeatmap} we show the heatmaps for the Head Contextualization version of ProtoContext on Aircraft's test set.  In both the 1-shot and 5-shot setting we see that the monolithic strategy of focusing on a single contextualizer is broken up when we use attention loss, and especially in the 5-shot setting we see the formation of a regime in which each sample focuses on the contextualizer for its class. In Figure~\ref{fig:miniImagenettHeatmap} we show the heatmaps the Head Contextualization version of ProtoContext on Mini-ImageNet's test set. Although we do not see quite as clear a regime as in Aircraft, it is still obvious that the monolithic attention strategy is broken up by our Attention Loss. Figure~\ref{fig:TieredImagenetHeatmap} shows a similiar picture to Figure~\ref{fig:aircraftHeatmap}, where we see a breakup of the single-contextualizer focus strategy in the 1-shot setting and a clear trend towards focusing on the contextualizer for a sample's class in the 5-shot setting. We do not include heatmaps for the Omniglot dataset because even without the Attention Loss, ProtoContext shows a strong focus on the contextualizer corresponding to each sample's class (see the main paper).
\begin{figure}
\caption{Heatmaps of the attention weights for the Aircraft dataset on the query set. Top row is 1 shot, bottom row is 5-shot, left column is without the attention loss, right column is with the attention loss. }
\begin{minipage}[t]{1.0\linewidth}
\centering
  \includegraphics[width=0.45\textwidth]{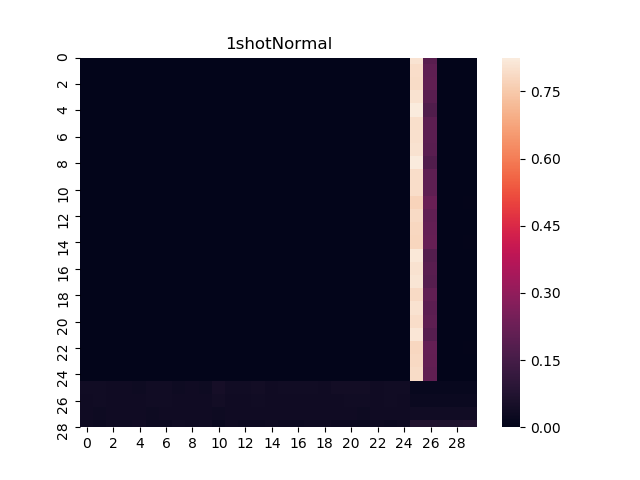}
  \includegraphics[width=0.45\textwidth]{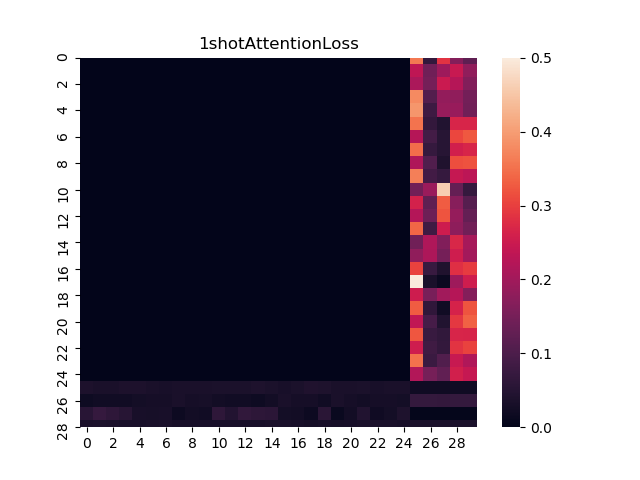}
   \includegraphics[width=0.45\textwidth]{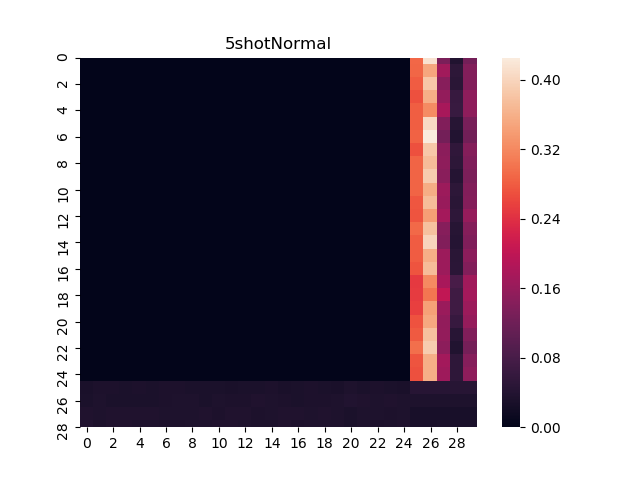}
  \includegraphics[width=0.45\textwidth]{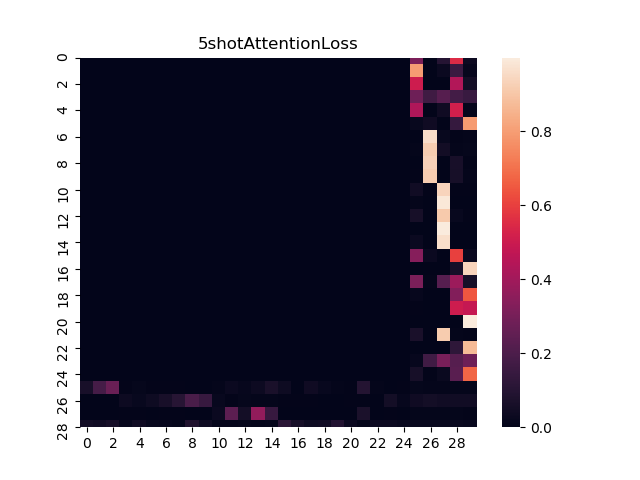}

\end{minipage}

\label{fig:aircraftHeatmap}
\end{figure}

\begin{figure}
\caption{Heatmaps of the attention weights for Mini-Imagenet on the query set. Top row is 1 shot, bottom row is 5 shot, left column is without the attention loss, right column is with the attention loss.}
\begin{minipage}[b]{1.0\linewidth}
\centering
  \includegraphics[width=0.45\textwidth]{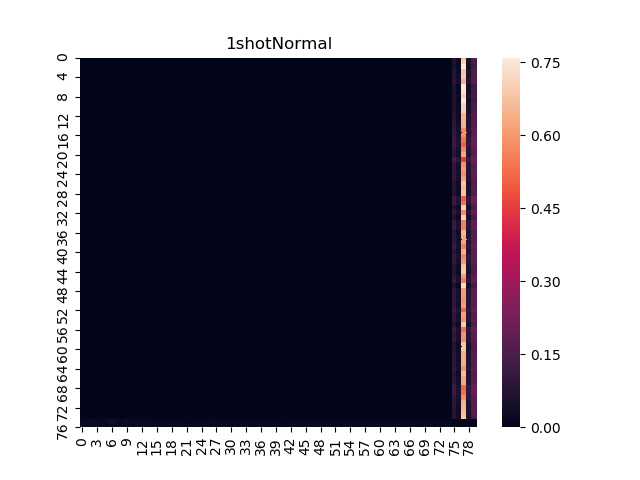}
  \includegraphics[width=0.45\textwidth]{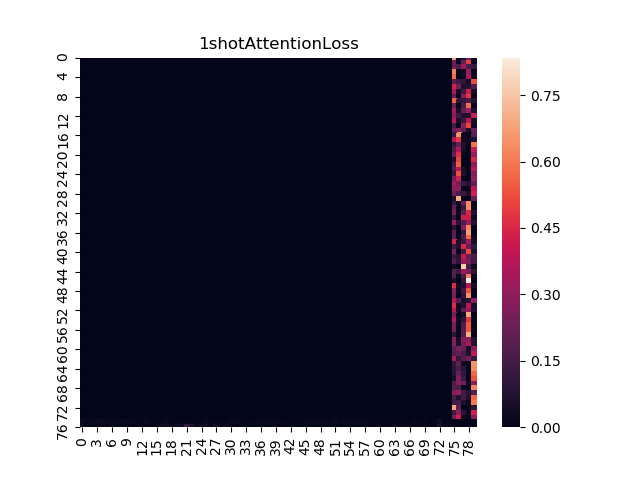}
   \includegraphics[width=0.45\textwidth]{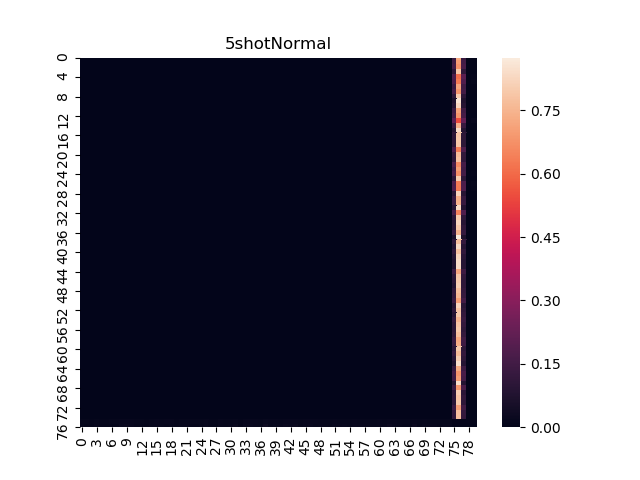}
  \includegraphics[width=0.45\textwidth]{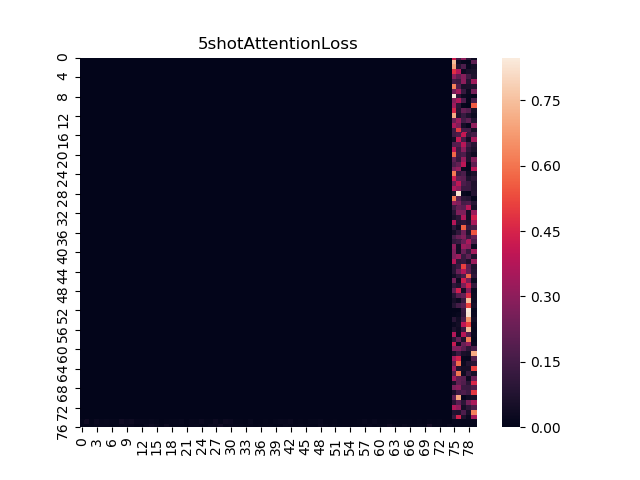}
\end{minipage}

\label{fig:miniImagenettHeatmap}
\end{figure}

\begin{figure}
\caption{Heatmaps of the attention weights for Tiered-ImageNet on the query set. Top row is 1 shot, bottom row is 5 shot, left column is without the attention loss, right column is with the attention loss.}
\begin{minipage}[b]{1.0\linewidth}
\centering
  \includegraphics[width=0.45\textwidth]{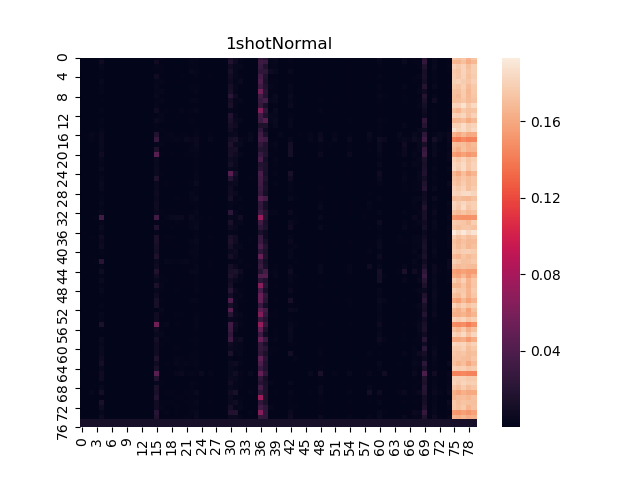}
  \includegraphics[width=0.45\textwidth]{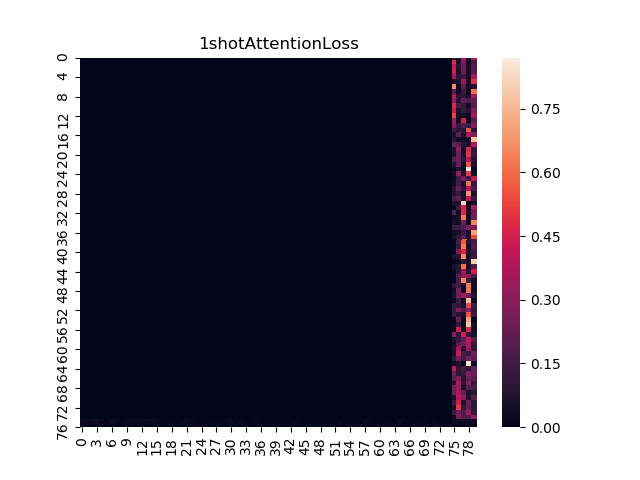}
   \includegraphics[width=0.45\textwidth]{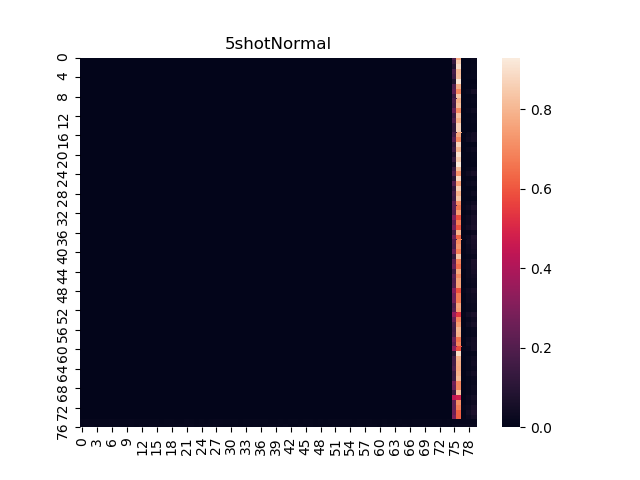}
  \includegraphics[width=0.45\textwidth]{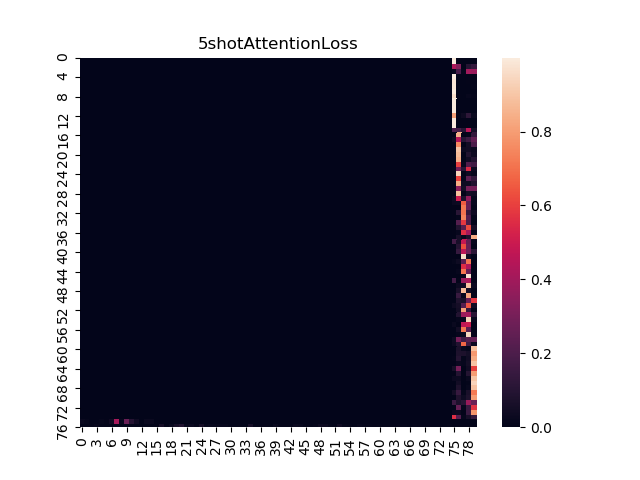}
\end{minipage}

\label{fig:TieredImagenetHeatmap}

\end{figure}


\section{Datasets}
\label{sec:datasets}
We experiment on four standard few shot classification datasets: Omniglot, Mini-Imagenet, Tiered-Imagenet, and FGCV Aircraft (Airplanes).
\paragraph{Omniglot}
The Omniglot dataset~\citep{lake2015human} consists of 1623 handwritten characters from 50 different alphabets. Within each alphabet, every character of the alphabet corresponds to a unique class and has several examples. In some uses of the Omniglot dataset \citep{antoniou2019, finn2017}, a task is created by sampling classes across alphabets with no regard to the structure of the dataset. This makes examples from different classes fairly distinct, allowing for simple use of prototypical representations for classification. We use a version more similar to that presented by \citet{lake2015human} which increases the similarity between classes in tasks sampled from the dataset, creating harder classification problems that requires more complex class representations, such as contextualizers. To realize this form of the dataset, we construct a sampling regime that incorporates the structure of the Omniglot dataset. Instead of sampling classes across all alphabets, we first select an alphabet uniformly at random and then select classes from that alphabet. If we assume that classes from the same alphabet will be more similar to each other than those from different alphabets, this regime increases the difficulty of the tasks drawn. We analyze results on 20-way classification, as with fewer ways all models solve the dataset nearly perfectly and there are no meaningful differences to be seen. 

\paragraph{Mini-Imagenet}
Mini-Imagenet \citep{vinyals2016} is a frequently used benchmark for meta learning. It is a subset of the Imagenet dataset containing 100 classes from the Imagenet dataset with 600 samples per class. We use the same splits as \citet{antoniou2019}.

\paragraph{Tiered-Imagenet}
Tiered Imagenet \citep{ren2018meta}  is also a classic benchmark in meta learning. It is made of 608 classes grouped into 34 high level sets based on the Imagenet hierarchy, 28 for training, 6 for validation, and 6 for testing. We believe
that by grouping similar classes together, harder tasks are produced. We note that this is an environment particularly apt for our contextualizer-based model, as the model is able to incorporate relations between the contextualizer and the samples to evaluate samples that are highly similar.

\paragraph{Aircraft}
The FGVC Airplanes benchmark \citep{maji2013fine} is, like Tiered Imagenet, a more fine-grained classification benchmark. It consists of 102 different aircraft model variants with 100 images of each. The dataset also has two coarser groupings of airplanes into ``Families'' and ``Manufacturers,'' however we disregard these in favor of the more challenging fine-grain task that uses variants as classes.

On all datasets, we measure performance for 1-shot and 5-shot learning. With the exception of  Omniglot, we perform all experiments with 5 ways. All experiments are performed on three separate seeds with the average result reported.

\section{Xavier Initialization}
\label{sec:xavier}
Although it is not proposed in~\citep{triantafillou2019}, we find empirically that scaling the last layer initialization used in ProtoMAML, ProtoContext, and PCABaseline to the magnitudes of \emph{Xavier initialization} proposed in \citet{Glorot2010UnderstandingTD} bolstered convergence and in some cases made convergence possible for our models. Specifically, if the initializations for the weights and bias are respectively $\mathbf{H}_w$ and $\mathbf{H}_b$, the feature size is $d_\mathrm{feature}$ and the output classes are $n,$ we compute as follows
\begin{equation*}
f = \frac{1}{\max(\mathbf{H}_b)}\sqrt{\frac{6}{d_\mathrm{feature}+n}}
\label{scalar}
\end{equation*}
and then rescale elementwise the intializations as $f \cdot \mathbf{H}_w$ and $f \cdot \mathbf{H}_b$ respectively.
\section{Training Specifications}
\label{sec:training}
In all models, our feature extractor is the same VGG~\citep{simonyan2014convolutional} architecture used in MAML++. In the case of ProtoMAML, ProtoContext, and PCA Baseline the head of the architecture is initialized in the manner detailed in the main paper. This means the head is not meta trained, as it is re-initialized at the beginning of each task. For our ProtoContext and Context Only models, we make use of the Transformer architecture and learning rate scheduler introduced in \citep{vaswani2017attention} for our attention mechanism.  Although we experimented with implementation of a full multi-head attention mechanism, we found that dropping some pieces of it improved performance. Our final results are reported using dot-product self-attention, and then feeding the results of that self-attention through a  layernorm layer and a feedforward layer with a skip connection around our attention mechanism and around our feedforward layer. We do not make use of multiple attention heads or stack several sub-layers instead opting for a single attention head that attends to samples and contextualizers. We present additional hyperparameters of interest in Table ~\ref{tab:hyper}. We direct the reader to the configuration files in our code submission for additional hyperparameters.
\begin{table}[h]
\caption{Additional hyperparameters.}
\begin{tabular}{lllll}
\toprule
Hyperparameter                        & Omniglot & Mini-ImageNet & Tiered-ImageNet & Aircraft \\
\cmidrule(r){1-1}
\cmidrule(r){2-2}
\cmidrule(r){3-3}
\cmidrule(r){4-4}
\cmidrule(r){5-5}
Key Dimension                  & 64       & 64            & 64              & 64       \\
Value Dimension                & 64       & 1200          & 1200            & 1200     \\
Warmup Steps (LR Scheduler)                  & 750      & 750           & 750             & 750      \\
Number of Attention Mechanisms & 1        & 1             & 1               & 1        \\
Number of Layers               & 1        & 1             & 1               & 1        \\
Inner Loop LR                  & 0.1      & 0.01          & 0.01            & 0.1      \\
Initial Outer Loop LR          & 0.001    & 0.001         & 0.001           & 0.001    \\
Feature Size                   & 64       & 1200          & 1200            & 1200     \\ 
\bottomrule
\end{tabular}
\label{tab:hyper}
\end{table}
\section{Additional Results: Quickdraw}
\label{sec:quickdraw}
A special challenge to our ProtoContext method, which is very dependent on class representation, is the Quickdraw dataset \citep{ha2017neural}. The dataset is constructed from images that thousands of users created in the Google Quick, Draw! challenge. Users were given a class and then given 20 seconds to draw an example of it. We believe that although each sample in a class is independent of each other one, they are not necessarily identically distributed as different user's drawing styles produce different distributions. As a result, creating a good representation for classes in this datset is a very difficult task. Nonetheless, in an attempt to challenge our model we train on Quickdraw and report results in Table~\ref{tab:quickdraw}. We outperform both ProtoMAML and MAML by about 1\% in 1 shot and 2\% in 5 shot. 
\begin{table}[!h]
  \small
  \caption{Accuracy for the Quickdraw dataset. All experiments done 5-way.}
  \centering
  \begin{tabular}{lcccc}
    \toprule
    \multicolumn{1}{c}{Shots}
    & \multicolumn{1}{c}{ProtoContext  (\emph{Contex. Prototypes})}
    & \multicolumn{1}{c}{ProtoContext (\emph{Head})} 
    & \multicolumn{1}{c}{ProtoMAML}
    & \multicolumn{1}{c}{MAML}\\
    \cmidrule(r){1-1}
    \cmidrule(r){2-2}
    \cmidrule(r){3-3}
    \cmidrule(r){4-4}
    \cmidrule(r){5-5}
   
    1 & $75.51\pm 1.15$ & $\mathbf{76.21} \pm 1.74$ & $73.36 \pm 0.62$ & $75.15 \pm 0.95$  \\
    5 &  $\mathbf{87.37} \pm 1.04$ & $86.87 \pm 0.86$ & $86.15 \pm 1.37$& $85.13 \pm 0.55$  \\

    \bottomrule
  \end{tabular}
  \label{tab:quickdraw}
\end{table}
\section{Additional Inner Loop Steps}
In Figure~\ref{fig:innerloop} we show plots of the additional inner loop steps for more datasets. We show the average across 3 random seeds. For MAML and ProtoMAML especially, standard deviations can be large, so we include them in a separate table below rather than clutter the figure with error bars (see Table~\ref{tab:std}). We note that for Mini-ImageNet 5-shot, as we add inner loop steps, both of our ProtoContext models increase in accuracy. This observation leads us to conclude the ProtoContext was not able to create a strong enough initialization in this dataset, instead settling on one that required additional fine-tuning. This may help us to understand why we do not obtain better results than MAML and Prototypical Networks in this setting.
\label{sec:innerloop}
\begin{figure}
    \caption{Plots of additional inner loop steps for each of the datasets, averaged across three random seeds. The legend is in the top left plot and we avoid repeating it for readability.}
\begin{minipage}[b]{1.0\linewidth}
\centering
    \includegraphics[width=0.45\textwidth]{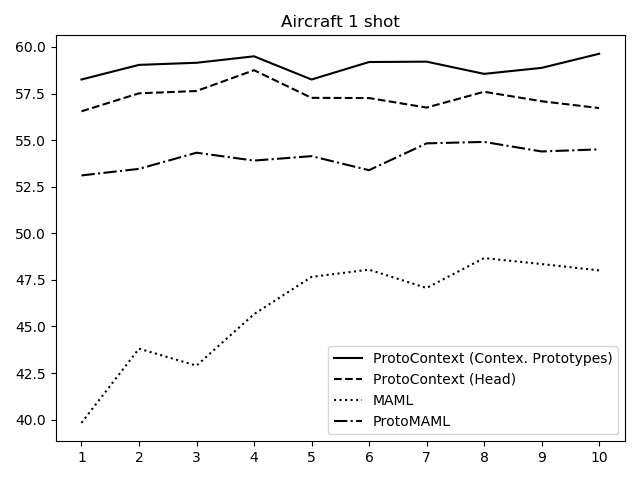}
    \includegraphics[width=0.45\textwidth]{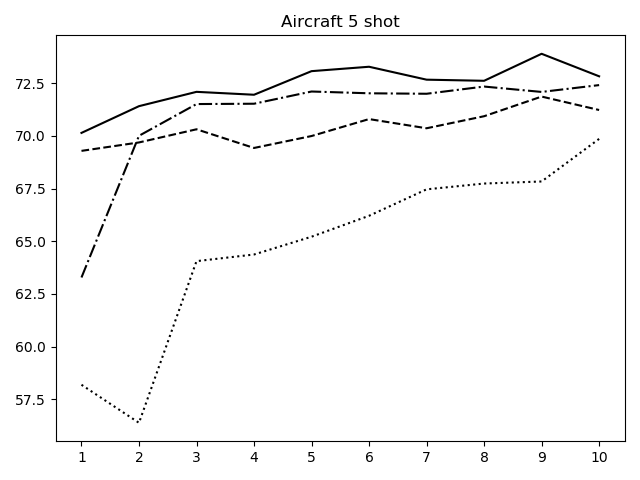}
    \includegraphics[width=0.45\textwidth]{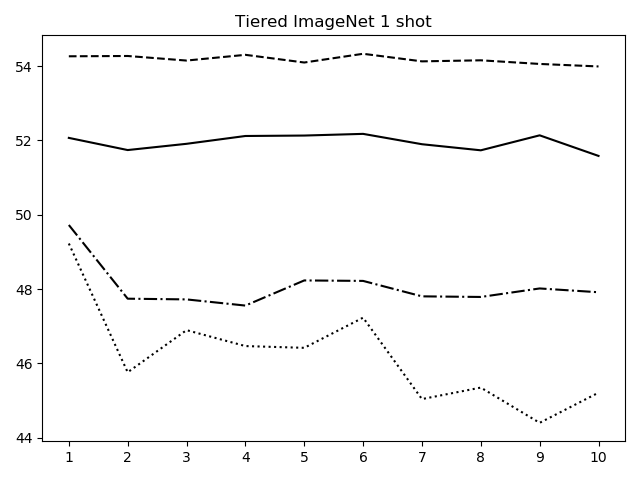}
    \includegraphics[width=0.45\textwidth]{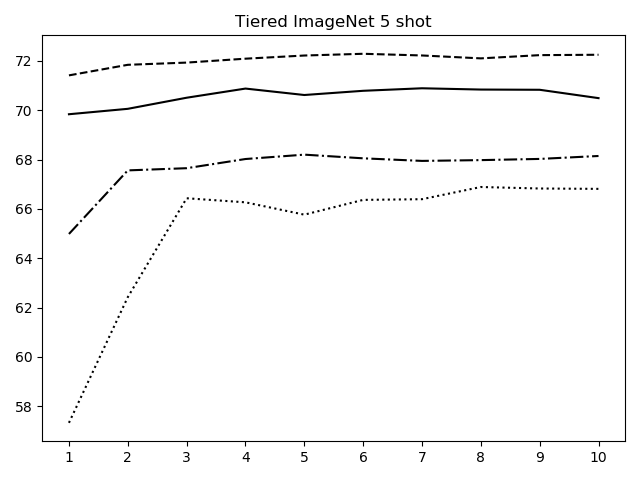}
    \includegraphics[width=0.45\textwidth]{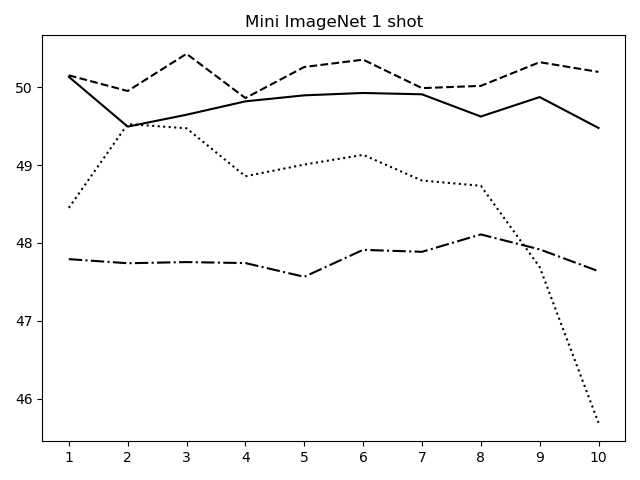}
    \includegraphics[width=0.45\textwidth]{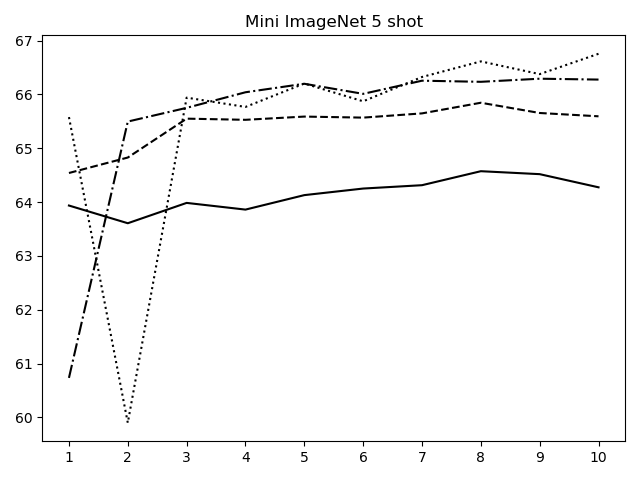}
    \includegraphics[width=0.45\textwidth]{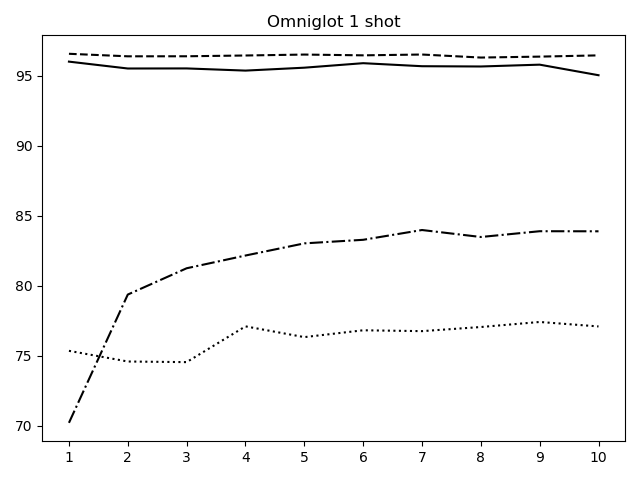}
    \includegraphics[width=0.45\textwidth]{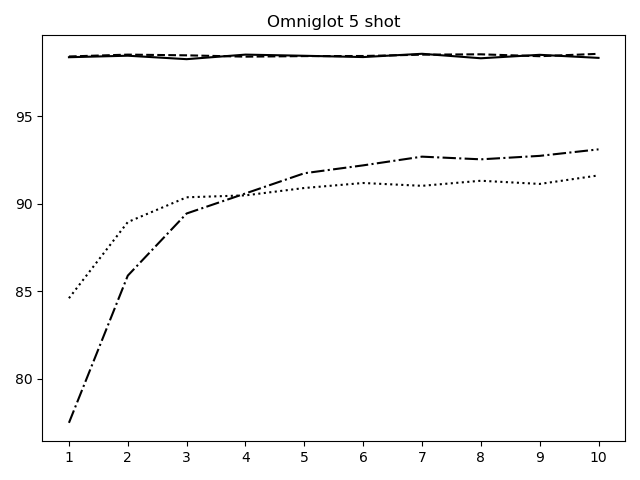}

\end{minipage}

    \label{fig:innerloop}
\end{figure}

\begin{table}[th]
  \tiny
  \centering
  \caption{Standard deviation for experiments on varying the inner loop steps.}
  \begin{tabular}{lcccccccccc}
    \toprule
    \multicolumn{11}{c}{Aircraft 1-shot} \\
    \midrule
    & \multicolumn{10}{c}{Steps} \\
    \cmidrule(r){2-11}
    \multicolumn{1}{c}{Method}
    & \multicolumn{1}{c}{1}
    & \multicolumn{1}{c}{2}
    & \multicolumn{1}{c}{3}
    & \multicolumn{1}{c}{4}
    & \multicolumn{1}{c}{5}
    & \multicolumn{1}{c}{6}
    & \multicolumn{1}{c}{7}
    & \multicolumn{1}{c}{8}
    & \multicolumn{1}{c}{9}
    & \multicolumn{1}{c}{10} \\
    \cmidrule(r){1-1}
    \cmidrule(r){2-2}
    \cmidrule(r){3-3}
    \cmidrule(r){4-4}
    \cmidrule(r){5-5}
    \cmidrule(r){6-6}
    \cmidrule(r){7-7}
    \cmidrule(r){8-8}
    \cmidrule(r){9-9}
    \cmidrule(r){10-10}
    \cmidrule(r){11-11}
    MAML & 6.95 & 0.76 & 1.27 & 0.95 & 0.87 & 1.10 & 1.99 & 1.45 & 1.78 & 1.22 \\
    ProtoMAML & 0.79 & 0.59 & 1.03 & 0.74 & 0.35 & 0.80 & 0.75 & 0.69 & 0.35 & 0.28 \\
    ProtoContext (\emph{Head}) & 0.15 & 1.23 & 0.83 & 1.36 & 0.72 & 0.61 & 0.34 & 0.45 & 0.10 & 0.37 \\
    ProtoContext (\emph{Contex. Prototypes}) & 0.12 & 0.44 & 0.54 & 0.63 & 0.72 & 0.42 & 0.22 & 0.53 & 0.64 & 1.68 \\
    \midrule
    \multicolumn{11}{c}{Aircraft 5-shot} \\
    \midrule
    & \multicolumn{10}{c}{Steps} \\
    \cmidrule(r){2-11}
    \multicolumn{1}{c}{Method}
    & \multicolumn{1}{c}{1}
    & \multicolumn{1}{c}{2}
    & \multicolumn{1}{c}{3}
    & \multicolumn{1}{c}{4}
    & \multicolumn{1}{c}{5}
    & \multicolumn{1}{c}{6}
    & \multicolumn{1}{c}{7}
    & \multicolumn{1}{c}{8}
    & \multicolumn{1}{c}{9}
    & \multicolumn{1}{c}{10} \\
    \cmidrule(r){1-1}
    \cmidrule(r){2-2}
    \cmidrule(r){3-3}
    \cmidrule(r){4-4}
    \cmidrule(r){5-5}
    \cmidrule(r){6-6}
    \cmidrule(r){7-7}
    \cmidrule(r){8-8}
    \cmidrule(r){9-9}
    \cmidrule(r){10-10}
    \cmidrule(r){11-11}
    MAML & 7.44 & 2.27 & 2.10 & 2.57 & 1.21 & 2.14 & 2.13 & 1.35 & 1.62 & 0.21 \\
    ProtoMAML & 0.47 & 0.60 & 0.20 & 0.76 & 0.41 & 0.49 & 0.62 & 0.31 & 0.31 & 0.24 \\
    ProtoContext (\emph{Head}) & 0.26 & 0.24 & 0.31 & 0.73 & 0.21 & 0.21 & 0.40 & 0.13 & 1.57 & 0.66 \\
    ProtoContext (\emph{Contex. Prototypes}) & 0.50 & 0.33 & 0.35 & 0.39 & 1.10 & 0.51 & 0.66 & 0.59 & 0.28 & 0.57 \\
    \midrule
    \multicolumn{11}{c}{Mini-Imagenet 1-shot} \\
    \midrule
    & \multicolumn{10}{c}{Steps} \\
    \cmidrule(r){2-11}
    \multicolumn{1}{c}{Method}
    & \multicolumn{1}{c}{1}
    & \multicolumn{1}{c}{2}
    & \multicolumn{1}{c}{3}
    & \multicolumn{1}{c}{4}
    & \multicolumn{1}{c}{5}
    & \multicolumn{1}{c}{6}
    & \multicolumn{1}{c}{7}
    & \multicolumn{1}{c}{8}
    & \multicolumn{1}{c}{9}
    & \multicolumn{1}{c}{10} \\
    \cmidrule(r){1-1}
    \cmidrule(r){2-2}
    \cmidrule(r){3-3}
    \cmidrule(r){4-4}
    \cmidrule(r){5-5}
    \cmidrule(r){6-6}
    \cmidrule(r){7-7}
    \cmidrule(r){8-8}
    \cmidrule(r){9-9}
    \cmidrule(r){10-10}
    \cmidrule(r){11-11}
    MAML & 2.38 & 0.39 & 0.52 & 1.63 & 0.49 & 0.38 & 0.23 & 0.42 & 1.32 & 0.28 \\
    ProtoMAML & 0.20 & 0.49 & 0.15 & 0.78 & 0.33 & 0.92 & 0.25 & 0.27 & 0.31 & 0.37 \\
    ProtoContext (\emph{Head}) & 0.17 & 0.03 & 0.19 & 0.23 & 0.26 & 0.38 & 0.21 & 0.31 & 0.09 & 0.15 \\
    ProtoContext (\emph{Contex. Prototypes}) & 0.92 & 0.32 & 0.16 & 0.21 & 0.38 & 0.25 & 0.16 & 0.53 & 0.49 & 0.71 \\
    \midrule
    \multicolumn{11}{c}{Mini-Imagenet 5-shot} \\
    \midrule
    & \multicolumn{10}{c}{Steps} \\
    \cmidrule(r){2-11}
    \multicolumn{1}{c}{Method}
    & \multicolumn{1}{c}{1}
    & \multicolumn{1}{c}{2}
    & \multicolumn{1}{c}{3}
    & \multicolumn{1}{c}{4}
    & \multicolumn{1}{c}{5}
    & \multicolumn{1}{c}{6}
    & \multicolumn{1}{c}{7}
    & \multicolumn{1}{c}{8}
    & \multicolumn{1}{c}{9}
    & \multicolumn{1}{c}{10} \\
    \cmidrule(r){1-1}
    \cmidrule(r){2-2}
    \cmidrule(r){3-3}
    \cmidrule(r){4-4}
    \cmidrule(r){5-5}
    \cmidrule(r){6-6}
    \cmidrule(r){7-7}
    \cmidrule(r){8-8}
    \cmidrule(r){9-9}
    \cmidrule(r){10-10}
    \cmidrule(r){11-11}
    MAML & 0.31 & 2.56 & 0.21 & 0.41 & 0.50 & 0.46 & 0.46 & 0.48 & 0.67 & 0.16 \\
    ProtoMAML & 0.81 & 0.22 & 0.58 & 0.66 & 0.27 & 0.42 & 0.31 & 0.40 & 0.03 & 0.30 \\
    ProtoContext (\emph{Head}) & 0.32 & 0.59 & 0.37 & 0.44 & 0.38 & 0.47 & 0.46 & 0.28 & 0.32 & 0.38 \\
    ProtoContext (\emph{Contex. Prototypes}) & 0.14 & 0.25 & 0.02 & 0.11 & 0.59 & 0.16 & 0.16 & 0.35 & 0.22 & 0.25 \\
    \midrule
    \multicolumn{11}{c}{Tiered-Imagenet 1-shot} \\
    \midrule
    & \multicolumn{10}{c}{Steps} \\
    \cmidrule(r){2-11}
    \multicolumn{1}{c}{Method}
    & \multicolumn{1}{c}{1}
    & \multicolumn{1}{c}{2}
    & \multicolumn{1}{c}{3}
    & \multicolumn{1}{c}{4}
    & \multicolumn{1}{c}{5}
    & \multicolumn{1}{c}{6}
    & \multicolumn{1}{c}{7}
    & \multicolumn{1}{c}{8}
    & \multicolumn{1}{c}{9}
    & \multicolumn{1}{c}{10} \\
    \cmidrule(r){1-1}
    \cmidrule(r){2-2}
    \cmidrule(r){3-3}
    \cmidrule(r){4-4}
    \cmidrule(r){5-5}
    \cmidrule(r){6-6}
    \cmidrule(r){7-7}
    \cmidrule(r){8-8}
    \cmidrule(r){9-9}
    \cmidrule(r){10-10}
    \cmidrule(r){11-11}
    MAML & 0.54 & 0.42 & 1.15 & 0.69 & 0.59 & 0.56 & 0.48 & 1.24 & 0.72 & 0.77 \\
    ProtoMAML & 0.50 & 0.45 & 0.59 & 0.58 & 0.37 & 0.60 & 0.34 & 0.58 & 0.18 & 0.81 \\
    ProtoContext (\emph{Head}) & 0.26 & 0.11 & 0.57 & 0.48 & 0.39 & 0.23 & 0.20 & 0.08 & 0.59 & 0.15 \\
    ProtoContext (\emph{Contex. Prototypes}) & 0.87 & 0.23 & 0.61 & 0.36 & 0.32 & 0.21 & 0.21 & 0.25 & 0.38 & 0.58 \\
    \midrule
    \multicolumn{11}{c}{Tiered-Imagenet 5-shot} \\
    \midrule
    & \multicolumn{10}{c}{Steps} \\
    \cmidrule(r){2-11}
    \multicolumn{1}{c}{Method}
    & \multicolumn{1}{c}{1}
    & \multicolumn{1}{c}{2}
    & \multicolumn{1}{c}{3}
    & \multicolumn{1}{c}{4}
    & \multicolumn{1}{c}{5}
    & \multicolumn{1}{c}{6}
    & \multicolumn{1}{c}{7}
    & \multicolumn{1}{c}{8}
    & \multicolumn{1}{c}{9}
    & \multicolumn{1}{c}{10} \\
    \cmidrule(r){1-1}
    \cmidrule(r){2-2}
    \cmidrule(r){3-3}
    \cmidrule(r){4-4}
    \cmidrule(r){5-5}
    \cmidrule(r){6-6}
    \cmidrule(r){7-7}
    \cmidrule(r){8-8}
    \cmidrule(r){9-9}
    \cmidrule(r){10-10}
    \cmidrule(r){11-11}
    MAML & 5.61 & 2.46 & 0.53 & 0.54 & 0.56 & 0.50 & 0.23 & 0.49 & 0.30 & 0.09 \\
    ProtoMAML & 0.25 & 0.41 & 0.28 & 0.33 & 0.31 & 0.27 & 0.14 & 0.25 & 0.35 & 0.28 \\
    ProtoContext (\emph{Head}) & 0.26 & 0.31 & 0.04 & 0.36 & 0.11 & 0.25 & 0.28 & 0.35 & 0.20 & 0.41 \\
    ProtoContext (\emph{Contex. Prototypes}) & 0.20 & 0.16 & 0.38 & 0.39 & 0.30 & 0.04 & 0.10 & 0.27 & 0.24 & 0.12 \\
        \multicolumn{11}{c}{Omniglot 1-shot} \\
    \midrule
    & \multicolumn{10}{c}{Steps} \\
    \cmidrule(r){2-11}
    \multicolumn{1}{c}{Method}
    & \multicolumn{1}{c}{1}
    & \multicolumn{1}{c}{2}
    & \multicolumn{1}{c}{3}
    & \multicolumn{1}{c}{4}
    & \multicolumn{1}{c}{5}
    & \multicolumn{1}{c}{6}
    & \multicolumn{1}{c}{7}
    & \multicolumn{1}{c}{8}
    & \multicolumn{1}{c}{9}
    & \multicolumn{1}{c}{10} \\
    \cmidrule(r){1-1}
    \cmidrule(r){2-2}
    \cmidrule(r){3-3}
    \cmidrule(r){4-4}
    \cmidrule(r){5-5}
    \cmidrule(r){6-6}
    \cmidrule(r){7-7}
    \cmidrule(r){8-8}
    \cmidrule(r){9-9}
    \cmidrule(r){10-10}
    \cmidrule(r){11-11}
    MAML & 1.78 & 2.12 & 2.75 & 3.13 & 2.43 & 2.49 & 3.09 & 3.46 & 2.76 & 2.66 \\
    ProtoMAML & 4.14 & 8.82 & 8.85 & 8.28 & 7.75 & 7.67 & 7.52 & 7.77 & 7.74 & 7.47 \\
    ProtoContext (\emph{Head}) & 0.01 & 0.17 & 0.23 & 0.01 & 0.20 & 0.33 & 0.15 & 0.27 & 0.18 & 0.04 \\
    ProtoContext (\emph{Contex. Prototypes}) & 0.32 & 0.19 & 0.33 & 0.01 & 0.13 & 0.16 & 0.37 & 0.49 & 0.49 & 0.78 \\
    \midrule
    \multicolumn{11}{c}{Omniglot 5-shot} \\
    \midrule
    & \multicolumn{10}{c}{Steps} \\
    \cmidrule(r){2-11}
    \multicolumn{1}{c}{Method}
    & \multicolumn{1}{c}{1}
    & \multicolumn{1}{c}{2}
    & \multicolumn{1}{c}{3}
    & \multicolumn{1}{c}{4}
    & \multicolumn{1}{c}{5}
    & \multicolumn{1}{c}{6}
    & \multicolumn{1}{c}{7}
    & \multicolumn{1}{c}{8}
    & \multicolumn{1}{c}{9}
    & \multicolumn{1}{c}{10} \\
    \cmidrule(r){1-1}
    \cmidrule(r){2-2}
    \cmidrule(r){3-3}
    \cmidrule(r){4-4}
    \cmidrule(r){5-5}
    \cmidrule(r){6-6}
    \cmidrule(r){7-7}
    \cmidrule(r){8-8}
    \cmidrule(r){9-9}
    \cmidrule(r){10-10}
    \cmidrule(r){11-11}
    MAML & 1.96 & 1.77 & 1.17 & 1.54 & 1.23 & 0.98 & 0.84 & 1.28 & 1.20 & 1.12 \\
    ProtoMAML & 3.90 & 1.84 & 2.76 & 2.29 & 1.34 & 1.55 & 1.49 & 1.44 & 1.50 & 1.34 \\
    ProtoContext (\emph{Head}) & 0.02 & 0.03 & 0.05 & 0.09 & 0.08 & 0.14 & 0.07 & 0.05 & 0.06 & 0.07 \\
    ProtoContext (\emph{Contex. Prototypes}) & 0.12 & 0.06 & 0.18 & 0.08 & 0.29 & 0.16 & 0.12 & 0.22 & 0.034 & 0.03 \\
    \bottomrule
  \end{tabular}
  \label{tab:std}
\end{table}

\clearpage

\end{document}